\documentclass{article} % For LaTeX2e
\usepackage{nips13submit_e,times}
\usepackage{hyperref}
\usepackage{url}
\usepackage{graphicx}
\usepackage{amssymb,amsmath,amsthm}
\usepackage{algorithmic,algorithm}
\usepackage{etex}
\usepackage{etoolbox}
\usepackage{cite}
\newcommand{\citet}{\cite}
\newcommand{\citep}{\cite}
\usepackage{multibib}
\newcites{sup}{Supplementary References}

\input{proofsatend.tex}
\def\x{{\mathbf x}}
\def\z{{\mathbf z}}
\def\1{{\mathbf 1}}

\def\ton{\underset{n \to +\infty}{\longrightarrow}}

\def\barf{{ \bar f}}
\def\barh{{ \bar h}}
\def\bartheta{{ \bar \theta}}
\def\hattheta{{ \hat \theta}}

\def\barg{{ \bar g}}

\def\tildej{\tilde{\jmath}}

\def\as{~~\text{a.s.}}

\def\alphab{{\boldsymbol\alpha}}

\def\D{{\mathbf D}}

\def\XX{{\mathcal X}}

\def\d{{\mathbf d}}
\def\E{{\mathbb E}}

\def\GG{{\mathcal G}}

\def\S{{\mathcal S}}

\def\d{{\mathbf d}}

\def\Real{{\mathbb R}}

\def\FF{{\mathcal F}}

\def\KK{{\mathcal K}}

\def\argmin{\operatornamewithlimits{arg\,min}}
\def\liminf{\operatornamewithlimits{lim\,inf}}

\def\trace{\operatorname{Tr}}

\def\defin{\triangleq}

\newcommand{\proofstep}[1]{\noindent{\bfseries #1:}~\newline}

\long\def\symbolfootnote[#1]#2{\begingroup\def\thefootnote{\fnsymbol{footnote}}\footnote[#1]{#2}\endgroup}

\newtheorem{theorem}{Theorem}[section]
\newtheorem{lemma}{Lemma}[section]
\newtheorem{proposition}{Proposition}[section]

\newtheorem{corollary}{Corollary}[section]
\newtheorem{definition}{Definition}[section]

\def\mybullet{~\hspace*{0.5cm}$\bullet$\hspace*{0.12cm}}
\def\begincondeq{\ifthenelse{\isundefined{\supplemental}}{
\begin{displaymath}
}{
\begin{equation}
}}
\def\endcondeq{\ifthenelse{\isundefined{\supplemental}}{
\end{displaymath}
}{
\end{equation}
}}
\def\condvspace{\ifthenelse{\isundefined{\supplemental}}{\vspace*{-0.1cm}}{}}
\def\condvspacesmall{\ifthenelse{\isundefined{\supplemental}}{\vspace*{-0.05cm}}{}}
\newcommand{\myvspace}[1]{\ifthenelse{\isundefined{\supplemental}}{\vspace*{-#1cm}}{}}

\fixstatement{theorem}
\fixstatement{lemma}
\fixstatement{proposition}
\fixstatement{corollary}

\nipsfinalcopy

\title{Stochastic Majorization-Minimization Algorithms for Large-Scale Optimization}

\author{
Julien Mairal \\
LEAR Project-Team - INRIA Grenoble\\
\texttt{julien.mairal@inria.fr}
}

\begin{document}

\maketitle

\begin{abstract}
Majorization-minimization algorithms consist of iteratively minimizing a
majorizing surrogate of an objective function. Because of its simplicity and
its wide applicability, this principle has been very popular in statistics and
in signal processing. In this paper, we intend to make this principle scalable.
We introduce a stochastic majorization-minimization scheme which is
able to deal with large-scale or possibly infinite data sets. When applied to
convex optimization problems under suitable assumptions, we show that it
achieves an expected convergence rate of $O(1/\sqrt{n})$ after~$n$ iterations,
and of $O(1/n)$ for strongly convex functions. Equally important, our scheme
almost surely converges to stationary points for a large class of non-convex
problems. We develop several efficient algorithms based on our framework. First, we
propose a new stochastic proximal gradient method, which experimentally matches
state-of-the-art solvers for large-scale $\ell_1$-logistic regression. Second,
we develop an online DC programming algorithm for non-convex sparse estimation.
Finally, we demonstrate the effectiveness of our approach for solving
large-scale structured matrix factorization problems.

\end{abstract}

\section{Introduction}
Majorization-minimization~\citep{lange2} is a simple optimization principle for
minimizing an objective function. It consists of iteratively minimizing a
surrogate that upper-bounds the objective, thus monotonically driving the
objective function value downhill. This idea is used in many existing
procedures. For instance, the expectation-maximization (EM) algorithm~(see
\citep{cappe,neal}) builds a surrogate for a likelihood model by using Jensen's
inequality.  Other approaches can also be interpreted under the
majorization-minimization point of view, such as DC programming~\citep{gasso}, where ``DC'' stands for difference of
convex functions, variational Bayes techniques~\citep{wainwright2}, or proximal
algorithms~\cite{beck,nesterov,wright}.

In this paper, we propose a stochastic majorization-minimization algorithm,
which is is suitable for solving large-scale problems arising in machine
learning and signal processing. More precisely, we address the minimization of
an expected cost---that is, an objective function that can be represented by an
expectation over a data distribution. For such objectives, online techniques
based on stochastic approximations have proven to be particularly efficient,
and have drawn a lot of attraction in machine learning, statistics, and
optimization~\citep{bottou,bottou2,cappe,duchi2,ghadimi,hu,hazan2,hazan3,lan,langford,leroux,mairal7,nemirovski,shalev,shalev2,shwartz,xiao}.

Our scheme follows this line of research. It consists of iteratively building a
surrogate of the expected cost when only a single data point is observed at
each iteration; this data point is used to update the surrogate, which in turn
is minimized to obtain a new estimate. Some previous works are closely related
to this scheme: the online EM algorithm for latent data
models~\citep{cappe,neal} and the online matrix factorization technique
of~\citet{mairal7} involve for instance surrogate functions updated in a
similar fashion.  Compared to these two approaches, our method is targeted to
more general optimization problems. 

Another related work is the incremental majorization-minimization
algorithm of~\cite{mairal17} for finite training sets; it was indeed shown to
be efficient for solving machine learning problems where storing dense
information about the past iterates can be afforded. Concretely, this
incremental scheme requires to store $O(pn)$ values, where $p$ is the
variable size, and $n$ is the size of the training set.\footnote{To alleviate
this issue, it is possible to cut the dataset into
$\eta$ mini-batches, reducing the memory load to $O(p \eta)$, which remains
cumbersome when $p$ is very large.} This issue was the main motivation for us for
proposing a stochastic scheme with a memory load independent of $n$, thus
allowing us to possibly deal with infinite data sets, or a huge variable size~$p$.

We study the convergence properties of our algorithm when the surrogates are
strongly convex and chosen among the class of \emph{first-order surrogate
functions} introduced in~\cite{mairal17}, which consist of approximating the
possibly non-smooth objective up to a smooth error. 
When the objective is convex, we obtain expected convergence rates that are
asymptotically optimal, or close to optimal~\cite{lan,nemirovski}. More precisely,
the convergence rate is of order $O(1/\sqrt{n})$ in a
finite horizon setting, and $O(1/n)$ for a strongly convex objective in an
infinite horizon setting. Our second analysis shows that for \emph{non-convex}
problems, our method
almost surely converges to a set of stationary points under suitable
assumptions. We believe that this result is equally valuable as convergence
rates for convex optimization.  To the best of our knowledge, the literature on
stochastic non-convex optimization is rather scarce, and we are only aware of
convergence results in more restricted settings than ours---see
for instance \cite{bottou2} for the stochastic gradient descent algorithm,
\cite{cappe} for online EM,~\cite{mairal7} for online matrix factorization,
or~\cite{ghadimi}, which provides stronger guarantees, but for unconstrained smooth problems.

We develop several efficient algorithms based on our framework. The first one is a
new stochastic proximal gradient method for composite or constrained
optimization. This algorithm is related to a long series of work in the convex optimization
literature~\cite{duchi2,hu,hazan3,lan,langford,nemirovski,shwartz,xiao},
and we demonstrate that it performs as well as
state-of-the-art solvers for large-scale $\ell_1$-logistic regression~\cite{fan2}.  The
second one is an online DC programming technique, which we demonstrate to be better than batch alternatives for large-scale non-convex
sparse estimation~\cite{gasso}. Finally, we show that our scheme can address efficiently 
structured sparse matrix factorization problems in an online fashion, and offers new possibilities to~\cite{jenatton2,mairal7} such
as the use of various loss or regularization functions.

This paper is organized as follows: Section~\ref{sec:batch} introduces
first-order surrogate functions for batch optimization; Section~\ref{sec:stoch} is devoted to our stochastic approach
and its convergence analysis; Section~\ref{sec:exp} presents several 
applications and numerical experiments, and Section~\ref{sec:ccl}
concludes the paper.

\section{Optimization with First-Order Surrogate Functions}\label{sec:batch}
Throughout the paper, we are interested in the minimization of a continuous
function~$f: \Real^p \to \Real$:
\begin{equation}
   \min_{\theta \in \Theta} f(\theta),\label{eq:obj}
\end{equation} 
where $\Theta \subseteq \Real^p$ is a convex set. The
majorization-minimization principle consists of computing a majorizing surrogate $g_n$ of $f$ at iteration $n$ and updating the current estimate by
$\theta_n \in \argmin_{\theta \in \Theta} g_n(\theta)$. The success
of such a scheme depends on how well the surrogates approximate~$f$. In this paper, we consider a particular class of surrogate functions introduced in~\cite{mairal17} and defined as follows:
\begin{definition}[\bfseries Strongly Convex First-Order Surrogate Functions]~\label{def:surrogate_batch}\newline
Let $\kappa$ be in $\Theta$. We denote by~$\S_{L,\rho}(f,\kappa)$ the set of
$\rho$-strongly convex functions $g$ such that $g \geq f$,
$g(\kappa)=f(\kappa)$, the approximation error $g-f$ is
differentiable, and the gradient $\nabla(g-f)$ is $L$-Lipschitz continuous.
We call the functions $g$ in $\S_{L,\rho}(f,\kappa)$ ``first-order surrogate functions''.
\end{definition}
Among the first-order surrogate functions presented in~\cite{mairal17}, we should mention the following ones:\\
\hspace*{-0.5cm}{\mybullet}{\bfseries Lipschitz Gradient Surrogates.}\\
When $f$ is differentiable and $\nabla f$ is $L$-Lipschitz, $f$ admits the following surrogate $g$ in $\S_{2L,L}(f,\kappa)$:
\begin{displaymath}
    g: \theta \mapsto f(\kappa) + \nabla f(\kappa)^\top (\theta-\kappa) + \frac{L}{2}\|\theta-\kappa\|_2^2.
\end{displaymath}
When $f$ is convex, $g$ is in $\S_{L,L}(f,\kappa)$, 
and when $f$ is $\mu$-strongly convex, $g$ is in $\S_{L-\mu,L}(f,\kappa)$.
Minimizing $g$ amounts to performing a classical classical gradient descent
step $\theta \leftarrow \kappa - \frac{1}{L}\nabla f(\kappa)$. \\
\hspace*{-0.5cm}{\mybullet}{\bfseries Proximal Gradient Surrogates.}\\
Assume that $f$ splits into $f = f_1 + f_2$, where $f_1$ is differentiable, $\nabla f_1$ is
$L$-Lipschitz, and $f_2$ is convex. Then, the function $g$ below is in
$\S_{2L,L}(f,\kappa)$:
\begin{displaymath}
    g: \theta \mapsto f_1(\kappa) + \nabla f_1(\kappa)^\top (\theta-\kappa) + \frac{L}{2}\|\theta-\kappa\|_2^2 + f_2(\theta).
\end{displaymath}
When $f_1$ is convex, $g$ is in $\S_{L,L}(f,\kappa)$. If $f_1$ is $\mu$-strongly convex, $g$ is in $\S_{L-\mu,L}(f,\kappa)$. 
Minimizing~$g$ amounts to a proximal gradient step~\cite{nesterov,beck,wright}: $\theta \leftarrow \argmin_{\theta} \frac{1}{2}\|\kappa-\frac{1}{L}\nabla f_1(\kappa) - \theta\|_2^2 + \frac{1}{L} f_2(\theta)$.

\hspace*{-0.5cm}{\mybullet}{\bfseries DC Programming Surrogates.}\\
Assume that $f = f_1 + f_2$, where $f_2$ is concave and differentiable, $\nabla f_2$ is $L_2$-Lipschitz, and~$g_1$ is in $\S_{L_1,\rho_1}(f_1,\kappa)$, Then, the following function $g$ is a surrogate in $\S_{L_1+L_2,\rho_1}(f,\kappa)$:
\begin{displaymath}
   g: \theta \mapsto f_1(\theta) + f_2(\kappa) + \nabla f_2(\kappa)^\top (\theta-\kappa).
\end{displaymath}
When $f_1$ is convex, $f_1+f_2$ is a difference of convex functions, leading to a DC program~\cite{gasso}.

With the definition of first-order surrogates and a basic ``batch'' algorithm
in hand, we now introduce our main contribution: a stochastic
scheme for solving large-scale problems.

\section{Stochastic Optimization}\label{sec:stoch}
As pointed out in~\cite{bottou}, one is usually not interested in the
minimization of an \emph{empirical cost} on a finite training set, but instead
in minimizing an \emph{expected cost}. Thus, we assume from now on that $f$ has the form of an expectation:
\begin{equation}
   \min_{\theta \in \Theta} \left[ f(\theta) \defin \E_{\x} [ \ell(\x,\theta) ]  \right], \label{eq:obj2}
\end{equation}
where $\x$ from some set $\XX$ represents a data point, which is drawn according to
some unknown distribution, and $\ell$ is a continuous loss function.  As often
done in the literature~\cite{nemirovski}, we assume
that the expectations are well defined and finite valued; we also assume that $f$ is bounded below.

We present our approach for tackling~(\ref{eq:obj2}) in Algorithm~\ref{alg:stochastic}. 
At each iteration, we draw a training point~$\x_n$, assuming that these points
are i.i.d. samples from the data distribution. Note that in practice, since it
is often difficult to obtain true i.i.d. samples, the points $\x_n$ are computed by
cycling on a randomly permuted training set~\cite{bottou}.
Then, we choose a surrogate $g_n$ for the function $\theta
\mapsto \ell(\x_n,\theta)$, and we use it to update a function 
$\barg_n$ that behaves as an approximate surrogate for the
expected cost~$f$. The function $\barg_n$ is in fact a weighted average of previously computed surrogates,
and involves a sequence of weights $(w_n)_{n \geq 1}$ that will
be discussed later. Then, we minimize $\barg_n$, and obtain a new estimate
$\theta_n$. For convex problems, we also propose to use averaging schemes, denoted by ``option 2'' and ``option 3'' in Alg.~\ref{alg:stochastic}. Averaging is a classical technique for improving
convergence rates in convex optimization~\citep{hazan3,nemirovski} for reasons
that are clear in the convergence proofs.

\begin{algorithm}[hbtp]
    \caption{Stochastic Majorization-Minimization Scheme}\label{alg:stochastic}
    \begin{algorithmic}[1]
    \INPUT $\theta_0 \in \Theta$ (initial estimate); $N$ (number of iterations); $(w_n)_{n \geq 1}$, weights in $(0,1]$; 
 \STATE initialize the approximate surrogate: $\barg_0: \theta \mapsto \frac{\rho}{2}\|\theta-\theta_0\|_2^2$; $\bartheta_0=\theta_0$; $\hattheta_0=\theta_0$;
    \FOR{ $n=1,\ldots,N$}
    \STATE draw a training point $\x_n$; define $f_n: \theta \mapsto \ell(\x_n,\theta)$;
    \STATE choose a surrogate function $g_n$ in $\S_{L,\rho}(f_n,\theta_{n-1})$;
    \STATE update the approximate surrogate: $\barg_n = (1-w_n) \barg_{n-1} + w_n g_n$;
    \STATE update the current estimate:
    \begin{displaymath}
       \theta_n \in \argmin_{\theta \in \Theta} \barg_n(\theta); % \label{eq:surrogate_batch}$
    \end{displaymath}
       \STATE for option 2, update the averaged iterate: $\hattheta_{n} \defin (1-w_{n+1})\hattheta_{n-1} + w_{n+1} \theta_{n}$;
       \STATE for option 3, update the averaged iterate: $\bartheta_{n} \defin \frac{(1-w_{n+1})\bartheta_{n-1} + w_{n+1} \theta_{n}}{\sum_{k=1}^{n+1}w_k}$;
    \ENDFOR
    \OUTPUT {\bfseries (option 1):} $\theta_{N}$ (current estimate, no averaging);
    \OUTPUT {\bfseries (option 2):} $\bartheta_{N}$ (first averaging scheme);
    \OUTPUT {\bfseries (option 3):} $\hattheta_{N}$ (second averaging scheme).
    \end{algorithmic}
\end{algorithm}

We remark that Algorithm~\ref{alg:stochastic} is only practical when
the functions $\barg_n$ can be parameterized with a small number of
variables, and when they can be easily minimized over $\Theta$. Concrete
examples are discussed in Section~\ref{sec:exp}.
Before that, we proceed with the convergence analysis.

\subsection{Convergence Analysis - Convex Case}\label{subsec:convex}
First, We study the case of convex functions $f_n: \theta \mapsto \ell(\theta,\x_n)$, and make the following assumption:
\begin{itemize}
   \item[\bfseries (A)] for all $\theta$ in $\Theta$, the functions $f_n$ are $R$-Lipschitz continuous. Note that for convex functions, this is
   equivalent to saying that subgradients of $f_n$ are uniformly bounded by $R$.
\end{itemize}
Assumption {\bfseries (A)} is classical in the stochastic optimization
literature~\cite{nemirovski}.  Our first result shows that with the averaging
scheme corresponding to ``option 2'' in Alg.~\ref{alg:stochastic}, we obtain an expected convergence rate that makes explicit
the role of the weight sequence $(w_n)_{n \geq 1}$.
\begin{proposition}[\bfseries Convergence Rate]~\label{prop:convex1}\newline
   When the functions $f_n$ are convex, under assumption~\textbf{(A)}, and when $\rho=L$, we have
   \begin{equation}
      \E[f(\bartheta_{n-1}) - f^\star] \leq \frac{L \|\theta^\star-\theta_0\|_2^2 + \frac{R^2}{L}\sum_{k=1}^{n} w_{k}^2}{ 2\sum_{k=1}^n w_{k}} ~~~~~~~~\text{for all}~n \geq 1, \label{eq:rateconvex}
   \end{equation}
   where $\bartheta_{n-1}$ is defined in Algorithm~\ref{alg:stochastic}, $\theta^\star$ is a minimizer of $f$ on $\Theta$, and $f^\star \defin f(\theta^\star)$.
\end{proposition}
\proofatend
   According to Lemma~\ref{lemma:surrogates}, we have for all $n \geq 1$,
\begin{displaymath}
     w_{n} B_{n-1} \leq w_{n} f^\star + Lw_{n} A_{n-1} -  L w_{n} \xi_{n-1} + w_{n} C_{n-1}.
\end{displaymath}
By using the relations~(\ref{eq:Bn}), this is equivalent to
\begin{displaymath}
   B_{n-1}-B_{n} + w_{n} \E[f(\theta_{n-1})] \leq w_{n} f^\star + L(A_{n-1} - A_{n}) + C_{n-1}-C_{n} + \frac{(R w_{n})^2}{2L}.
\end{displaymath}
By summing these inequalities between $1$ and $n$, we obtain
\begin{displaymath}
   B_0 - B_{n} + \sum_{k=1}^n w_{k} \E[f(\theta_{k-1})] \leq \left(\sum_{k=1}^n w_{k}\right) f^\star + LA_0 - LA_{n} - C_{n} + \sum_{k=1}^n \frac{(R w_{k})^2}{2L}.
\end{displaymath}
   Note that we also have
   \begin{displaymath}
      B_{n} \leq f^\star + L A_{n} + C_{n} =  L A_{n}  + C_{n} + B_0 -LA_0 + L\xi_0.
   \end{displaymath}
   Therefore, by combining the two previous inequalities,
 \begin{displaymath}
    \sum_{k=1}^n w_{k} \E[f(\theta_{k-1})] \leq \left(\sum_{k=1}^n w_{k}\right)
    f^\star + L\xi_0 + \sum_{k=1}^n \frac{(R w_{k})^2}{2L},
\end{displaymath}
   and by using Jensen's inequality,
 \begin{displaymath}
    \E[f(\bartheta_{n-1}) - f^\star] \leq \frac{L \xi_0 +
    \frac{R^2}{2L}\sum_{k=1}^{n} w_{k}^2}{ \sum_{k=1}^n w_{k}}.
\end{displaymath}
\endproofatend
Such a rate is similar
to the one of stochastic gradient descent with averaging, see~\cite{nemirovski}
for example.  Note that the constraint $\rho = L$ here is compatible with the
proximal gradient surrogate. 

From Proposition~\ref{prop:convex1}, it is easy to obtain a $O(1/\sqrt{n})$
bound for a finite horizon---that is, when the total number of iterations~$n$ is known
in advance. When $n$ is fixed, such a bound can indeed be obtained by plugging
constant weights $w_k = {\gamma}/{\sqrt{n}}$ for all $k \leq n$
in~Eq.~(\ref{eq:rateconvex}).  Note that the upper-bound $O(1/\sqrt{n})$ cannot
be improved in general without making further assumptions on the objective
function~\cite{nemirovski}.
The next corollary shows that in an infinite horizon setting and with
decreasing weights, we lose a logarithmic factor compared to an optimal
convergence rate~\cite{lan,nemirovski} of $O(1/\sqrt{n})$.
\begin{corollary}[\bfseries Convergence Rate - Infinite Horizon - Decreasing Weights]~\label{corollary:infinite}\newline
Let us make the same assumptions as in Proposition~\ref{prop:convex1} and choose the weights~$w_n = \gamma/\sqrt{n}$. Then,
 \begin{displaymath}
 \E[f(\bartheta_{n-1}) - f^\star] \leq \frac{L \|\theta^\star-\theta_0\|_2^2}{2\gamma \sqrt{n}} + \frac{R^2\gamma(1+\log(n))}{2L\sqrt{n}},~~\forall n \geq 2.
\end{displaymath}
\end{corollary}
\proofatend~\newline
   We choose weights of the form $w_n \defin \frac{\gamma}{\sqrt{n}}$. Then, we have
   \begin{displaymath}
       \sum_{k=1}^n w_k^2 \leq \gamma^2(1+\log{n}),
   \end{displaymath}
   by using the fact that $\sum_{k=1}^n \frac{1}{k} \leq 1+ \log(n)$. We also have for $n\geq 2$,
   \begin{displaymath}
       \sum_{k=1}^n w_k \geq 2 \gamma (\sqrt{n+1}-1) \geq \gamma \sqrt{n},
   \end{displaymath}
   where we use the fact that $\sum_{k=1}^n \frac{1}{\sqrt{k}} \geq
   2(\sqrt{n+1}-1)$, and the fact that $2(\sqrt{n+1}-1) \geq \sqrt{n}$ for all
   $n \geq 2$. 
   Plugging this inequalities into~(\ref{eq:rateconvex}) yields the desired result.
\endproofatend
Our analysis suggests to use weights of the
form~$O(1/\sqrt{n})$. In practice, we have found that choosing
$w_n = \sqrt{{n_0+1}} / \sqrt{{n_0+n}}$ performs well, where $n_0$ is tuned on a subsample of the training
set.
%In the next section, we show that our scheme leads to an $O(1/n)$ rate for strongly convex functions.

\subsection{Convergence Analysis - Strongly Convex Case}\label{subsec:strongconvex}
In this section, we introduce an additional assumption:
\begin{itemize}
   \item[\bfseries (B)] the functions $f_n$ are $\mu$-strongly convex.
\end{itemize}
We show that our method achieves a rate $O(1/n)$, which is
optimal up to a multiplicative constant for strongly convex functions (see~\cite{lan,nemirovski}).
\begin{proposition}[\bfseries Convergence Rate]\label{prop:stronglyconvex}~\newline
   Under assumptions {\bfseries (A)} and {\bfseries (B)}, with $\rho=L+\mu$. 
   Define $\beta \defin \frac{\mu}{\rho}$ and $w_n  \defin \frac{1+\beta}{1+\beta n}$.
   Then, 
   \begin{displaymath}
      \E[f(\hattheta_{n-1})-f^\star] + \frac{\rho}{2}\E[\|\theta^\star-\theta_n\|_2^2]  \leq \max\left( \frac{2 R^2}{\mu} ,\rho\| \theta^\star-\theta_0\|_2^2\right) \frac{1}{\beta n + 1} ~~~\text{for all}~n\geq 1,
   \end{displaymath}
   where $\hattheta_n$ is defined in Algorithm~\ref{alg:stochastic}, when choosing the averaging scheme called ``option 3''.
%    \begin{displaymath}
%       \hattheta_n \defin (1-w_{n+1})\hattheta_{n-1} + w_{n+1} \theta_{n} ~~~\text{with}~~~ \hattheta_0 \defin \theta_0.
%    \end{displaymath}
\end{proposition}
\proofatend

We proceed in several steps, proving the convergence rates of several quantities of interest.

   \proofstep{Convergence rate of $C_n$}
   Let us show by induction that we have $C_n \leq \frac{R^2}{\rho}w_n$ for all
   $n\geq 1$. This is obviously true for $n=1$ by definitions of $w_1=1$ and
   $C_1=\frac{R^2}{2\rho}$.
   Let us now assume that it is true for $n-1$. We have
   \begin{equation}
      \begin{split}
         C_n & = (1-w_n)C_{n-1} + \frac{R^2}{2\rho}w_n^2 \\
             & \leq \frac{R^2}{\rho}w_n\left( (1-w_n)\frac{w_{n-1}}{w_n} + \frac{w_n}{2}  \right) \\
             & \leq \frac{R^2}{\rho}w_n\left( \frac{ \beta (n-1)}{\beta n + 1}\frac{\beta n + 1}{ \beta (n-1)+1} + \frac{1}{\beta n + 1}  \right) \\
             & \leq \frac{R^2}{\rho}w_n\left( \frac{ \beta (n-1)}{ \beta (n-1)+1} + \frac{1}{\beta (n-1) + 1}  \right) \\
             & = \frac{R^2}{\rho}w_n.  
      \end{split} \label{eq:Cn}
   \end{equation}
   We conclude by induction that this is true for all $n\geq 1$.

   \proofstep{Convergence rate of $A_n$}
   From Lemma~\ref{lemma:AnBn} and~\ref{lemma:surrogates}, we have for all $n\geq 2$,
   \begin{displaymath}
      \mu A_{n-1} \leq L A_{n-1} - \rho \xi_{n-1} + C_{n-1}.
   \end{displaymath}
   Multiplying this inequality by $w_n$,
   \begin{displaymath}
         2\mu w_n A_{n-1} \leq \rho w_n( A_{n-1} - \xi_{n-1}) + w_n C_{n-1},
   \end{displaymath}
   where the factor $2$ comes from the fact that $\rho=L+\mu$.
   By using the definition of $A_n$ in Eq.~(\ref{eq:Bn}), we obtain the relation
   \begin{displaymath}
      A_n \leq \left(1-\frac{2\mu w_n}{\rho}\right)A_{n-1} + \frac{w_n}{\rho} C_{n-1}.
   \end{displaymath}
   Let us now show by induction that we have, for all $n \geq 1$, the
   convergence rate $A_n \leq  \delta w_n$, where $\delta \defin \max\left( \frac{R^2}{\rho\mu} , \xi_0 \right)
   $. For $n=1$, we have that  and $w_1=1$, and thus $A_1 = \xi_0 \leq \delta$. 
   Assume now that we have $A_{n-1} \leq \delta w_{n-1}$ for some $n \geq 1$.
   Then, by using the convergence rate~(\ref{eq:Cn}) and the induction hypothesis, 
   \begin{displaymath}
      \begin{split}
         A_n & \leq \delta w_n \left( \left( 1-\frac{2\mu w_n}{\rho} \right) \frac{w_{n-1}}{w_n} + \frac{R^2 w_{n-1}}{\rho^2 \delta} \right) \\
             & \leq \delta w_n \left( \left( 1-\frac{2\mu w_n}{\rho} \right) \frac{w_{n-1}}{w_n} + \mu \frac{w_{n-1}}{\rho} \right)  \\
             & \leq \delta w_n   \left( \frac{\beta n + 1 - \frac{2\mu(1+\beta)}{\rho}}{\beta n + 1} \frac{\beta n + 1}{\beta (n-1) + 1} + \frac{\frac{\mu(1+\beta)}{\rho}}{\beta (n-1) + 1} \right) \\
             & = \delta w_n  \left( \frac{\beta n + 1 - \frac{\mu(1+\beta)}{\rho}}{\beta (n-1) + 1} \right) \\
             & \leq \delta w_n.
      \end{split}
   \end{displaymath}
   The last inequality uses the fact that $\frac{\mu(1+\beta)}{\rho} \geq \beta$ because $\beta \leq \frac{\mu}{L}$.
   we conclude by induction that $A_n \leq \delta w_n$ for all $n \geq 1$.

\proofstep{Convergence rate of $\E[f(\hattheta_n)-f^\star] + \rho \xi_n$}
We use again Lemma~\ref{lemma:surrogates}:
\begin{displaymath}
   B_n -f^\star + \rho \xi_n \leq LA_n + C_n,
\end{displaymath}
and we consider two possible cases
\begin{itemize}
   \item If $\frac{R^2}{\rho\mu} \geq \xi_0$, then 
\begin{displaymath}
   \begin{split}
      B_n -f^\star + \rho \xi_n & \leq \frac{R^2}{\rho}\left( 1 + \frac{L}{\mu} \right)w_n \\
       & = \frac{R^2}{\mu} w_n \\
       & \leq \frac{2R^2}{\mu (\beta n + 1)},
   \end{split}
\end{displaymath}
where we simply use the convergence rates of $A_n$ and $C_n$ computed before.

\item If instead $\frac{R^2}{\rho\mu} < \xi_0$, then 
\begin{displaymath}
   \begin{split}
      B_n -f^\star + \rho \xi_n & \leq \left( \frac{R^2}{\rho} + L \xi_0 \right)w_n \\
       & \leq \rho \xi_0 w_n \\
       & \leq \frac{2\rho \xi_0}{\beta n + 1}.
   \end{split}
\end{displaymath}
It is then easy to prove that $\E[f(\hattheta_n)-f^\star] \leq B_n$ by using Jensen's inequality, which allows us to conclude.
\end{itemize}

\endproofatend
The averaging scheme is slightly different than in the previous section and the
weights decrease at a different speed. Again, this rate applies to the proximal
gradient surrogates, which satisfy the constraint $\rho=L+\mu$.
In the next section, we analyze our scheme in a non-convex setting.

\subsection{Convergence Analysis - Non-Convex Case}\label{subsec:nonconvex}
Convergence results for non-convex problems are by nature weak, and difficult
to obtain for stochastic optimization~\cite{bottou,ghadimi}. In such a context, proving
convergence to a global (or local) minimum is out of reach, and classical
analyses study instead asymptotic stationary point conditions, which involve
directional derivatives (see~\cite{borwein,mairal17}).  Concretely, we
introduce the following assumptions:
\begin{itemize}
   \item[\bfseries (C)] $\Theta$ and the support~$\XX$ of the data are compact;
   \item[\bfseries (D)] The functions $f_n$ are uniformly bounded by some constant $M$;
   \item[\bfseries (E)] The weights $w_n$ are non-increasing, $w_1=1$, $\sum_{n \geq 1} w_n \!=\! +\infty$, and $\sum_{n \geq 1} w_n^2 \sqrt{n} \! < \! + \infty$;
   \item[\bfseries (F)] The directional derivatives $\nabla f_n(\theta,\theta'-\theta)$, and $\nabla f(\theta,\theta'-\theta)$  exist for all $\theta$ and $\theta'$ in~$\Theta$.
\end{itemize}
Assumptions~{\bfseries (C)} and~{\bfseries (D)} combined with~{\bfseries
(A)} are useful because they allow us to use some uniform convergence
results from the theory of empirical processes~\cite{vaart}.
In a nutshell, these assumptions ensure that the function class $\{ \x \mapsto
\ell(\x,\theta) : \theta \in \Theta \}$ is ``simple enough'', such that
a uniform law of large numbers applies. The assumption~{\bfseries (E)} is more
technical: it resembles classical conditions used for proving the convergence of
stochastic gradient descent algorithms, usually stating that the weights
$w_n$ should be the summand of a diverging sum while the sum of $w_n^2$ should
be finite; the constraint $\sum_{n \geq 1} w_n^2 \sqrt{n} < +\infty$ is 
slightly stronger. Finally,~{\bfseries (F)} is a mild assumption, which is useful to
characterize the stationary points of the problem.  A classical necessary
first-order condition~\cite{borwein} for~$\theta$ to be a local minimum of $f$
is indeed to have $\nabla f(\theta,\theta'-\theta)$ non-negative for
all~$\theta'$ in~$\Theta$.  We call such points $\theta$ the stationary points of
the function~$f$.
The next proposition is a generalization of a convergence result obtained
in~\cite{mairal7} in the context of sparse matrix factorization.

\begin{proposition}[\bfseries Non-Convex Analysis - Almost Sure Convergence]~\label{prop:nonconvex}\newline
   Under assumptions~{\bfseries (A)},~{\bfseries (C)},~{\bfseries (D)},~{\bfseries (E)},
   $(f(\theta_n))_{n \geq 0}$ converges with probability one. Under
   assumption~{\bfseries (F)}, we also have that 
\begin{displaymath}
  \liminf_{n \to +\infty} \inf_{\theta \in \Theta} \frac{\nabla
  \barf_n(\theta_{n},\theta-\theta_n)}{ \|\theta-\theta_n\|_2} \geq 0,
\end{displaymath}
where the function $\barf_n$ is a weighted empirical risk recursively defined
as $\barf_n = (1-w_n)\barf_{n-1} + w_n f_n$. It can be shown that $\barf_n$
uniformly converges to $f$.
\end{proposition}
\proofatend

We generalize the proof of convergence for online matrix factorization of~\cite{mairal7}. The proof exploits Theorem~\ref{theo:martingales} about the convergence of
quasi-martingales~\citesup{fisk}, similarly as~\citet{bottou2} for proving the convergence of the stochastic gradient descent algorithm 
for non-convex functions. 

\proofstep{Almost sure convergence of $(\barg_n(\theta_n))_{n \geq 1}$}
The first step consists of applying a convergence theorem for the sequence $(\barg_n(\theta_n))_{n \geq 1}$ by bounding its positive expected variations. Define $Y_n \defin \barg_n(\theta_n)$. For $n\geq 2$, we have
   \begin{equation}
       \begin{split}
          Y_n \!-\! Y_{n-1} & = \barg_n(\theta_n) \!-\!  \barg_n(\theta_{n-1}) \!+ \! \barg_n(\theta_{n-1}) \!-\! \barg_{n-1}(\theta_{n-1}) \\
                                                        & = (\barg_n(\theta_n) \!-\!  \barg_n(\theta_{n-1})) \!+\! w_n( g_n(\theta_{n-1})\!-\! \barg_{n-1}(\theta_{n-1}))  \\
                                                        & = (\barg_n(\theta_n) \!-\!  \barg_n(\theta_{n-1})) \!+\! w_n(\barf_{n-1}(\theta_{n-1}) \!-\! \barg_{n-1}(\theta_{n-1}))  \!+\!  w_n ( g_n(\theta_{n-1})\!-\! \barf_{n-1}(\theta_{n-1})) \\
                                                        & = (\barg_n(\theta_n) \!-\!  \barg_n(\theta_{n-1})) \!+\! w_n(\barf_{n-1}(\theta_{n-1}) \!-\! \barg_{n-1}(\theta_{n-1}))  \!+\!  w_n ( f_n(\theta_{n-1})\!-\! \barf_{n-1}(\theta_{n-1})) \\
                                                        & \leq w_n ( f_n(\theta_{n-1})\!-\! \barf_{n-1}(\theta_{n-1})).
       \end{split} \label{eq:tmp}
   \end{equation}
   The final inequality comes from the inequality $\barg_n \geq \barf_n$,
   which is easy to show by induction starting from $n=1$ since $w_1=1$.  It follows,
   \begin{displaymath}
   \begin{split}
        \E[ \barg_n(\theta_n) - \barg_{n-1}(\theta_{n-1}) | \FF_{n-1}] & \leq w_n \E[ f_n(\theta_{n-1}) -  \barf_{n-1}(\theta_{n-1}) | \FF_{n-1}] \\
                                                                       & = w_n ( f(\theta_{n-1}) -   \barf_{n-1}(\theta_{n-1})) \\
                                                                       & \leq w_n \sup_{\theta \in \Theta} | f(\theta) -  \barf_{n-1}(\theta) |,
   \end{split}
   \end{displaymath}
   where $\FF_{n-1}$ is the filtration representing the past information before time $n$.
   Call  now
   \begin{displaymath}
      \delta_n \defin  \left\{ 
      \begin{array}{ll} 
            1 & ~~\text{if}~~            \E[ \barg_n(\theta_n) - \barg_{n-1}(\theta_{n-1}) | \FF_{n-1}] > 0 \\
            0 & ~~\text{otherwise.} 
      \end{array}
   \right.
   \end{displaymath}
   Then, the series below with non-negative summands converges:
   \begin{displaymath}
   \begin{split}
      \sum_{n=1}^\infty \E[ \delta_n (\barg_n(\theta_n) - \barg_{n-1}(\theta_{n-1}))] & = \sum_{n=1}^\infty \E[ \delta_n \E[ (\barg_n(\theta_n) - \barg_{n-1}(\theta_{n-1}))| \FF_{n-1}]] \\
         & \leq \sum_{n=1}^\infty \E\left[ w_n \sup_{\theta \in \Theta} | f(\theta) -  \barf_{n-1}(\theta) |\right] \\
         & \leq \sum_{n=1}^\infty  C w_n^2 \sqrt{n} < +\infty,
   \end{split}
   \end{displaymath}
   The second inequality comes from Lemma~\ref{lemma:uniform}.
   Since in addition $\barg_n$ is bounded below by some constant independent of
   $n$, we can apply Theorem~\ref{theo:martingales}.  This theorem tells us that
   $(\barg_n(\theta_n))_{n \geq 1}$ converges almost surely to an integrable
   random variable $g^\star$ and that $\sum_{n=1}^\infty \E[|\E[\barg_n(\theta_n) - \barg_{n-1}(\theta_{n-1})|\FF_{n-1}]|]$ converges almost surely.

\proofstep{Almost sure convergence of $(\barf_n(\theta_n))_{n \geq 1}$}
We will show by using Lemma~\ref{lemma:converg} that the non-positive term $\barf_n(\theta_{n})-
\barg_{n}(\theta_{n})$ almost surely converges to zero, and thus
$(\barf_n(\theta_n))_{n \geq 1}$ is also converging almost surely to $g^\star$.

We observe that 
$$\sum_{n=1}^\infty \E[|\E[\barg_n(\theta_n) - \barg_{n-1}(\theta_{n-1})|\FF_{n-1}]|] = \E\left[\sum_{n=1}^\infty |\E[\barg_n(\theta_n) - \barg_{n-1}(\theta_{n-1})|\FF_{n-1}]|\right] < +\infty.$$
Thus, the series $\sum_{n=1}^\infty |\E[\barg_n(\theta_n) -
\barg_{n-1}(\theta_{n-1})|\FF_{n-1}]|$ is absolutely convergent with
probability one, and the series $\sum_{n=1}^\infty \E[\barg_n(\theta_n) -
\barg_{n-1}(\theta_{n-1})|\FF_{n-1}]$ is also almost surely convergent.

We also remark that, using Lemma~\ref{lemma:uniform},  
\begin{displaymath}
  \E\left[ \sum_{n=1}^{+\infty} w_n|f(\theta_{n-1})- \barf_{n-1}(\theta_{n-1})| \right] \leq C \sum_{n=1}^{+\infty} w_n^2 \sqrt{n} < +\infty, 
\end{displaymath}
and thus $w_n(f(\theta_{n-1})- \barf_{n-1}(\theta_{n-1}))$ is the summand of an
absolutely convergent series with probability one. 

Taking the expectation of Eq.~(\ref{eq:tmp}) conditioned on~$\FF_{n-1}$, it remains that the
non-positive term $w_n (\barf_{n-1}(\theta_{n-1})- \barg_{n-1}(\theta_{n-1}))$ is also
necessarily the summand of an almost surely convergent series, since all other terms in the equation are summands of almost surely converging sums.
This is not sufficient to immediately conclude that $\barf_n(\theta_{n})-
\barg_{n}(\theta_{n})$ converges to zero almost surely, and thus we will use
Lemma~\ref{lemma:converg}. We have that $\sum_{n=1}^{+\infty} w_n$ diverges,
that $\sum_{n=1}^{+\infty} w_n (\barg_{n-1}(\theta_{n-1})- \barf_{n-1}(\theta_{n-1}))$ converges almost surely. Define
$X_n \defin (\barg_{n-1}(\theta_{n-1})- \barf_{n-1}(\theta_{n-1}))$. By definition of the surrogate functions, the differences
$h_n \defin g_n-f_n$ are differentiable and their gradients are $L$-Lipschitz continuous.
Since in addition $\Theta$ is compact and $\nabla h_n(\theta_{n-1})=0$, $\nabla
h_n$ is bounded by some constant $R'$ independent of $n$, and the function
$h_n$ is $R'$-Lipschitz.  This is therefore also the case for $\barh_n =
\barg_n-\barf_n$.
 \begin{displaymath}
    \begin{split}
       |X_{n+1}-X_n|  & = |\barh_n(\theta_n)- \barh_{n-1}(\theta_{n-1})| \\
                      & \leq |\barh_n(\theta_n)-\barh_n(\theta_{n-1})| + |\barh_n(\theta_{n-1}) - \barh_{n-1}(\theta_{n-1}) | \\
                      & \leq R'\|\theta_n-\theta_{n-1}\|_2 + |\barh_n(\theta_{n-1}) - \barh_{n-1}(\theta_{n-1}) | \\
                      & \leq \frac{2R R'}{\rho} w_n + w_n|  h_n(\theta_{n-1}) - \barh_{n-1}(\theta_{n-1})| \\
                      & = \frac{2R R'}{\rho} w_n + w_n|\barh_{n-1}(\theta_{n-1})| \\
                      & \leq O(w_n).
    \end{split}
 \end{displaymath} 
 The second inequality uses the fact that $\barh_n$ is $R'$-Lipschitz; The
 second inequality uses Lemma~\ref{lemma:stability_simple}; the last equality
 uses the fact that the functions $h_n$ are also bounded by
 some constant independent of $n$ (using the fact that $\nabla h_n$ is
 uniformly bounded).  We can now apply Lemma~\ref{lemma:converg}, and $X_n$
 converges to zero with probability one.  Thus, $(\barf_n(\theta_n))_{n \geq
 1}$ converges almost surely to $g^\star$.

\proofstep{Almost sure convergence of $(f(\theta_n))_{n \geq 1}$}
Since $(\barf_n(\theta_n))_{n \geq 1}$ converges almost surely, we simply use
Lemma~\ref{lemma:converg2}, which tells us that $\barf_n$ converges uniformly to $f$. Then, $(f(\theta_n))_{n \geq 1}$ converges almost
surely to $g^\star$.

\proofstep{Asymptotic Stationary Point Condition}
Let us call $\barh_n \defin \barg_n-\barf_n$, which can be shown to be
differentiable with a $L$-Lipschitz gradient by definition of the surrogate $g_n$.  For all $\theta$ in $\Theta$,
\begin{displaymath}
    \nabla \barf_n(\theta_n,\theta-\theta_n) = \nabla
    \barg_n(\theta_n,\theta-\theta_n) - \nabla
    \barh_n(\theta_n)^\top(\theta-\theta_n).
\end{displaymath}
Since $\theta_n$ is the minimizer of $\barg_n$, we have $\nabla
\barg_n(\theta_n,\theta-\theta_n) \geq 0$.
   
Since $\barh_n$ is differentiable and its gradient is $L$-Lipschitz continuous, we can apply
Lemma~\ref{lemma:upperlipschitz} to $\theta=\theta_n$ and
$\theta'=\theta_n-\frac{1}{L}\nabla \barh_n(\theta_n)$, which gives
$\barh_n(\theta') \leq \barh_n(\theta_n)-\frac{1}{2L}\|\nabla
\barh_n(\theta_n)\|_2^2$.
Since we have shown that $\barh_n(\theta_n) = \barg_n(\theta_n)-\barf(\theta_n)$ converges to zero and $\barh_n(\theta') \geq 0$, 
we have that $\|\nabla \barh_n(\theta_n)\|_2$ converges to zero.
Thus,
\begin{displaymath}
   \inf_{\theta \in \Theta} \frac{\nabla \barf_n(\theta_n,\theta-\theta_n)}{ \|\theta-\theta_n\|_2  } \geq - \|\nabla \barh_n(\theta_n)\|_2  \ton 0 \as
\end{displaymath}

\endproofatend
Even though $\barf_n$ converges uniformly to the expected cost $f$,
Proposition~\ref{prop:nonconvex} does not imply that the limit
points of $(\theta_n)_{n \geq 1}$ are stationary points of $f$. 
We obtain such a guarantee when the surrogates that are parameterized, an
assumption always satisfied when Algorithm~\ref{alg:stochastic} is used in
practice.
\begin{proposition}[\bfseries Non-Convex Analysis - Parameterized Surrogates]~\label{prop:nonconvex2}\newline
   Let us make the same assumptions as in Proposition~\ref{prop:nonconvex}, and let us
   assume that the functions $\barg_n$ are parameterized by
   some variables $\kappa_n$ living in a compact set $\KK$ of $\Real^d$. In other words,
   $\barg_n$ can be written as $g_{\kappa_n}$, with $\kappa_n$ in $\KK$.
   Suppose there exists a constant $K > 0$ such that $|g_{\kappa}(\theta) -
   g_{\kappa'}(\theta)| \leq K \| \kappa-\kappa'\|_2$ for all $\theta$ in $\Theta$
   and $\kappa,\kappa'$ in $\KK$.
   Then, every limit point $\theta_{\infty}$ of the sequence
   $(\theta_n)_{n\geq 1}$ is a stationary point of $f$---that is, for all $\theta$ in $\Theta$,
\begin{displaymath}
\nabla f(\theta_{\infty}, \theta-\theta_{\infty}) \geq 0.
\end{displaymath}
\end{proposition}
\proofatend
Since $\Theta$ is compact according to assumption~{\bfseries (C)}, the sequence
$(\theta_n)_{n\geq 1}$ admits limit points.  Let us consider a converging
subsequence $(n_k)_{k \geq 1}$ to a limit point $\theta_{\infty}$ in $\Theta$.
In this converging subsequence, we can also find a subsequence $(n_{k'})_{k'
\geq 1}$ such that $\kappa_{n_{k'}}$ converges to a point $\kappa_{\infty}$ in
$\KK$ (which is compact).  For the sake of simplicity, and without loss of
generality, we remove the indices $k$ and $k'$ from the notation and assume
that $\theta_n$ converges to $\theta_{\infty}$, while $\kappa_n$ converges to
$\kappa_{\infty}$. It is then easy to see that the functions $\barg_n$ converge
uniformly to $\barg_{\infty} \defin g_{\kappa_\infty}$, given the assumptions made
in the proposition.

Defining $\barh_{\infty} \defin \barg_{\infty} - f$, we have for all $\theta$ in $\Theta$:
\begin{displaymath}
    \nabla f(\theta_\infty,\theta-\theta_\infty) = \nabla
    \barg_\infty(\theta_\infty,\theta-\theta_\infty) - \nabla
    \barh_\infty(\theta_\infty,\theta-\theta_\infty).
\end{displaymath}
To prove the proposition, we will first show that $\nabla \barg_\infty(\theta_\infty,\theta-\theta_\infty) \geq 0$ and then that 
$\nabla \barh_\infty(\theta_\infty,\theta-\theta_\infty)=0$.

\proofstep{Proof of $\nabla \barg_\infty (\theta_\infty,\theta-\theta_\infty)\geq 0$}
It is sufficient to show that $\theta_{\infty}$ is a minimizer of $\barg_{\infty}$. This is 
straightforward, by taking the limit when $n$ goes to infinity of
\begin{displaymath}
   \barg_n(\theta) \geq \barg_n(\theta_n),
\end{displaymath}
where we use the uniform convergence of $\barg_n$.

\proofstep{Proof of $\nabla \barh_\infty(\theta_\infty,\theta-\theta_\infty)=0$}
Since both $\barf_n$ and $\barg_n$ converges uniformly (according to
Lemma~\ref{lemma:uniform} for $\barf_n$), we have that $\barh_n$ converges
uniformly to $\barh_{\infty}$. Since $\barh_n$ is differentiable with a
$L$-Lipschitz gradient, we have for all vector $\z$ in $\Real^p$,
\begin{displaymath}
   \barh_{n}(\theta_n + \z) = \barh_n(\theta_n) + \nabla \barh_n(\theta_n)^\top \z + O( \|\z\|_2^2),
\end{displaymath} 
where the constant in $O$ is independent of $n$.
By taking the limit when $n$ goes to infinity, it remains
\begin{displaymath}
   \barh_{\infty}(\theta_\infty + \z) = \barh_\infty(\theta_\infty)  + O( \|\z\|_2^2),
\end{displaymath} 
since we have shown in the proof of Proposition~\ref{prop:nonconvex} that $\|\nabla
\barh_n(\theta_n)\|_2$ converges to zero.
Since~$\barh_\infty$ admits a first order extension around $\theta_\infty$ it is differentiable
at this point and furthermore, $\nabla \barh_\infty(\theta_\infty) = 0$.
This is sufficient to conclude.
\endproofatend

Finally, we show that our non-convex convergence analysis can be extended
beyond first-order surrogate functions---that is, when $g_n$ does not satisfy
exactly Definition~\ref{def:surrogate_batch}. This is possible when the
objective has a particular partially separable structure, as shown in the next
proposition.  This extension was motivated by the non-convex sparse estimation
formulation of Section~\ref{sec:exp}, where such a structure appears.
\begin{proposition}[\bfseries Non-Convex Analysis - Partially Separable Extension]~\label{prop:nonconvex3}\newline
   Assume that the functions $f_n$ split into $f_n(\theta) \!=\!
   f_{0,n}(\theta) \!+\! \sum_{k=1}^K f_{k,n}(\gamma_k(\theta))$, where the
   functions $\gamma_k\!:\! \Real^p \! \to \!\Real$ are convex and
   $R$-Lipschitz, and the $f_{k,n}$ are non-decreasing for $k \geq 1$.  Consider
   $g_{n,0}$ in $\S_{L_0,\rho_1}(f_{0,n},\theta_{n-1})$, and some non-decreasing functions $g_{k,n}$ 
   in $\S_{L_k,0}(f_{k,n},\gamma_k(\theta_{n-1}))$.
   Instead of choosing $g_n$ in $\S_{L,\rho}(f_n,\theta_{n-1})$ in Alg~\ref{alg:stochastic}, replace it by $g_n\! \defin\! \theta\! \mapsto \!g_{0,n}(\theta) \!+\! g_{k,n}(\gamma_k(\theta))$.

   Then, Propositions~\ref{prop:nonconvex} and~\ref{prop:nonconvex2} still hold.
\end{proposition}
% The function $g_n$ defined in this proposition is strongly convex, and
% we show in the supplemental material that the approximation error $g_n\!-\!f_n$ has enough regularity
% for the analysis to hold. %, despite the fact that $g_n \!-\! f_n$ is
%not differentiable in general.
\proofatend
First we notice that 
\begin{itemize}
   \item $g_n \geq f_n$;
   \item $g_n(\theta_{n-1}) = f_n(\theta_{n-1})$;
   \item $g_n$ is $\rho_1$-strongly convex since $\theta \mapsto g_{k,n}(\gamma_k(\theta))$ can be shown to be convex, following elementary composition rules for convex functions (see \mycite{boyd}, Section 3.2.4).
\end{itemize}
Thus, the only property missing is the smoothness of the approximation error
$h_n \defin g_n - f_n$.  Rather than writing again a full proof, we now simply
review the different places where this property is used, and which
modifications should be made to the proofs of Propositions~\ref{prop:nonconvex}
and~\ref{prop:nonconvex2}.

In the second step of this proof, we require the functions $h_n$ to be
uniformly Lipschitz and uniformly bounded. It easy to check that it is still
the case with the assumptions we made in Proposition~\ref{prop:nonconvex3}.

The last step about the asymptotic point condition is however more problematic,
where we cannot show anymore that the quantity $\nabla \barh_n(\theta_n)$
converges to zero (since $\barh_n$ is not differentiable anymore).  Instead, we
need to show that the directional derivative $\frac{\nabla \barh_n(\theta_n,
\theta-\theta_n)}{\|\theta-\theta_n\|}$ uniformly converges to zero on
$\Theta$.

We will show the result for $K=1$; it will be easy to extend it to any arbitrary~$K > 2$.
We remark that
\begin{displaymath}
   \nabla \barh_n(\theta_n,\theta-\theta_n) =    \nabla \barh_{0,n}(\theta_n)^\top(\theta-\theta_n) + \lim_{t \to 0^+} \frac{ \barh_{1,n}(\gamma_1(\theta_{n}+t(\theta-\theta_n))) - \barh_{1,n}(\gamma_1(\theta_n))}{t}, 
\end{displaymath}
where $\barh_{0,n}$ and $\barh_{1,n}$ are defined similarly as $\barh_n$ for the
functions $h_{0,n}\defin g_{0,n}-f_{0,n}$ and $h_{1,n}\defin g_{1,n}-f_{1,n}$ respectively.  Since $\barh_n(\theta_n)$ is shown
to converge to zero, we have that the non-negative quantities
$\barh_{0,n}(\theta_n)$ and $\barh_{1,n}(\gamma_1(\theta_n))$ converge to zero as well. Since~$\barh_{0,n}$ and $\barh_{1,n}$ are differentiable and their gradients are Lipschitz, we use similar
arguments as in the proof of Proposition~\ref{prop:nonconvex}, and we have that $\nabla \barh_{0,n}(\theta_n)$ and $\barh_{1,n}'(\gamma_1(\theta_n))$ converge to zero (where $\barh_{1,n}'$ is the derivative of $\barh_{1,n}$.
Concerning the second term, we can make the following Taylor expansion for $\barh_{1,n}$:
\begin{displaymath}
   \barh_{1,n}( \gamma_1(\theta_n+\z)) = \barh_{1,n}(\gamma_1(\theta_n)) + \barh_{1,n}^{\prime}(\gamma_1(\theta_n))(\gamma_1(\theta_{n}+\z) - \gamma_1(\theta_n)) + O\left( ( \gamma_1(\theta_{n}+\z) - \gamma_1(\theta_n))^2\right),
\end{displaymath}
where the constant in the $O$ notation is independent of $\theta_n$ and $\z$ (since the derivative is $L_1$-Lipschitz).
Plugging $\z \defin t (\theta-\theta_n)$ in this last equation, and using the Lipschitz property of $\gamma_1$, we have 
\begin{displaymath}
   \lim_{t \to 0^+} \left|\frac{ \barh_{1,n}( \gamma_1(\theta_{n}+t(\theta-\theta_n))) - \barh_{1,n}(\gamma_1(\theta_n))}{t} \right| \leq |\barh_{1,n}^{\prime}(\gamma_1(\theta_n))| \|\theta-\theta_n\|.
\end{displaymath}
Since~$\barh_{1,n}^{\prime}(\gamma_1(\theta_n))$ converges to zero, we can conclude
the proof of the modified Proposition~\ref{prop:nonconvex}.

The proof of Proposition~\ref{prop:nonconvex2} can be modified with similar arguments.
\endproofatend

\section{Applications and Experimental Validation}\label{sec:exp}
In this section, we introduce different applications, and provide 
numerical experiments. A C++/Matlab implementation is available in the software package 
SPAMS~\cite{mairal7}.\footnote{\url{http://spams-devel.gforge.inria.fr/}.}  All experiments were performed on a
single core of a 2GHz Intel CPU with $64$GB of RAM.

\subsection{Stochastic Proximal Gradient Descent Algorithm}\label{sec:logistic}
Our first application is a stochastic proximal gradient descent method,
which we call SMM (Stochastic Majorization-Minimization), for solving problems of the
form:
\begin{equation}
   \min_{\theta \in \Theta} \E_{\x}[\ell(\x,\theta)] + \psi(\theta),\label{eq:prox}
\end{equation}
where $\psi$ is a convex deterministic regularization function, and the
functions $\theta \mapsto \ell(\x,\theta)$ are differentiable and their
gradients are $L$-Lipschitz continuous. We can thus use
the proximal gradient surrogate presented in
Section~\ref{sec:batch}. Assume that a weight sequence $(w_n)_{n \geq 1}$ is chosen such that $w_1\!=\!1$.
By defining some other weights~$w_n^i$ recursively as $w_n^i
\!\defin\!(1\!-\!w_n)w_n^{i-1}$ for $i \!<\!  n$ and $w_n^n\! \defin\! w_n$,
our scheme yields the update rule:
\begin{equation}
    \theta_{n} \leftarrow \argmin_{\theta \in \Theta} \sum_{i=1}^n w_n^i \left[ \textstyle \nabla f_i(\theta_{i-1})^\top \theta + \frac{L}{2} \|\theta-\theta_{i-1}\|_2^2 + \psi(\theta)   \right].\tag{SMM}
\end{equation}
Our algorithm is related to FOBOS~\cite{duchi2},
to SMIDAS~\cite{shwartz} or the truncated
gradient method~\cite{langford} (when $\psi$ is the $\ell_1$-norm).  These three algorithms use indeed the
following update rule:
\begin{equation}
   \theta_{n} \leftarrow \argmin_{\theta \in \Theta} \textstyle \nabla f_n(\theta_{n-1})^\top \theta + \frac{1}{2\eta_n} \|\theta-\theta_{n-1}\|_2^2 + \psi(\theta),\tag{FOBOS}
\end{equation}
Another related scheme is the regularized dual averaging (RDA) of~\cite{xiao},
which can be written as
\begin{equation}
   \theta_{n} \leftarrow \argmin_{\theta \in \Theta} \frac{1}{n}\sum_{i=1}^n 
   \textstyle \nabla f_i(\theta_{i-1})^\top \theta   + \frac{1}{2\eta_n}
 \|\theta\|_2^2 + \psi(\theta).\tag{RDA}
\end{equation}
Compared to these approaches, our scheme includes a weighted average of previously seen gradients,
and a weighted average of the past iterates. Some links can also be drawn with approaches such as the
``approximate follow the leader'' algorithm of~\cite{hazan3} and other works~\cite{hu,lan}.

We now evaluate the performance of our method for $\ell_1$-logistic regression.
In summary, the datasets consist of pairs~$(y_i,\x_i)_{i=1}^N$, where the $y_i$'s are
in~$\{-1,+1\}$, and the $\x_i$'s are in $\Real^p$ with unit
$\ell_2$-norm. The function~$\psi$ in~(\ref{eq:prox}) is the $\ell_1$-norm:
$\psi(\theta) \defin \lambda \|\theta\|_1$, and $\lambda$ is a regularization
parameter; the functions $f_i$ are logistic losses: $f_i(\theta)\defin
\log(1+e^{-y_i \x_i^\top \theta})$. One part of each dataset is devoted to
training, and another part to testing. We used weights of the form $w_n
\!\defin\! \sqrt{{(n_0+1)}/(n+n_0)}$, where~$n_0$ is automatically adjusted at
the beginning of each experiment by performing one pass on $5\%$ of the training data.  We
implemented SMM in C++ and exploited the sparseness of the datasets,
such that each update has a computational complexity of the order
$O(s)$, where~$s$ is the number of non zeros in $\nabla f_n(\theta_{n-1})$;
such an implementation is non trivial but proved to be very efficient.

We consider three datasets described in the table below.
\textsf{rcv1} and \textsf{webspam} are obtained from the 2008 Pascal
large-scale learning
challenge.\footnote{\url{http://largescale.ml.tu-berlin.de}.} \textsf{kdd2010}
is available from the LIBSVM
website.\footnote{\url{http://www.csie.ntu.edu.tw/~cjlin/libsvm/}.} 
\begin{center}
\begin{tabular}{|l|c|c|c|c|c|}
\hline
name & $N_{\text{tr}}$ (train) & $N_{\text{te}}$ (test) & $p$ & density ($\%$) & size (GB) \\ 
\hline
\textsf{rcv1} & $781\,265$ & $23\,149$ & $47\,152$ & $0.161$ & $0.95$\\
\hline
\textsf{webspam} & $250\,000$ & $100\,000$ & $16\,091\,143$ & $0.023$ & $14.95$ \\
\hline
\textsf{kdd2010} & $10\,000\,000$ & $9\,264\,097$ & $28\,875\,157$ & $10^{-4}$ & $4.8$ \\
\hline
\end{tabular}
\end{center}

We compare our implementation with state-of-the-art publicly available solvers:
the batch algorithm FISTA of~\cite{beck} implemented in the C++ SPAMS
toolbox and LIBLINEAR v1.93~\cite{fan2}. LIBLINEAR is based on
a working-set algorithm, and, to the best of our knowledge, is one of the
most efficient available solver for $\ell_1$-logistic regression with sparse
datasets. Because~$p$ is large, the incremental majorization-minimization method of~\cite{mairal17} could
not run for memory reasons.
We run every method on $1,2,3,4,5,10$ and $25$ epochs (passes over the training set),
for three regularization regimes, respectively yielding a solution with
approximately $100, 1\,000$ and $10\,000$ non-zero coefficients. 
We report results for the medium regularization in
Figure~\ref{fig:exp_l1} and provide the rest as supplemental material. FISTA is not represented in this figure since it required
more than $25$ epochs to achieve reasonable values.
Our conclusion is that \emph{SMM often provides a reasonable
solution after one epoch, and outperforms LIBLINEAR in the low-precision
regime. For high-precision regimes, LIBLINEAR should be preferred.}
Such a conclusion is often obtained when comparing batch and stochastic algorithms~\cite{bottou},
but matching the performance of LIBLINEAR is very challenging.
\begin{figure}[hbtp]
   \centering
   %\vspace*{-0.25cm}
   \includegraphics[width=0.3\linewidth]{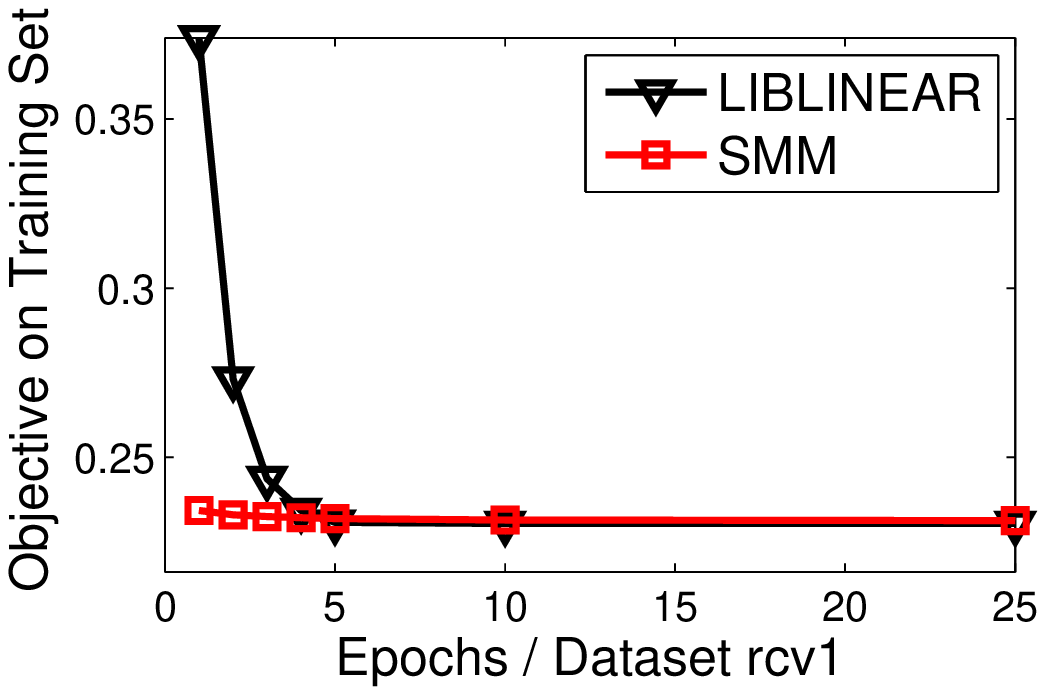}
   \includegraphics[width=0.3\linewidth]{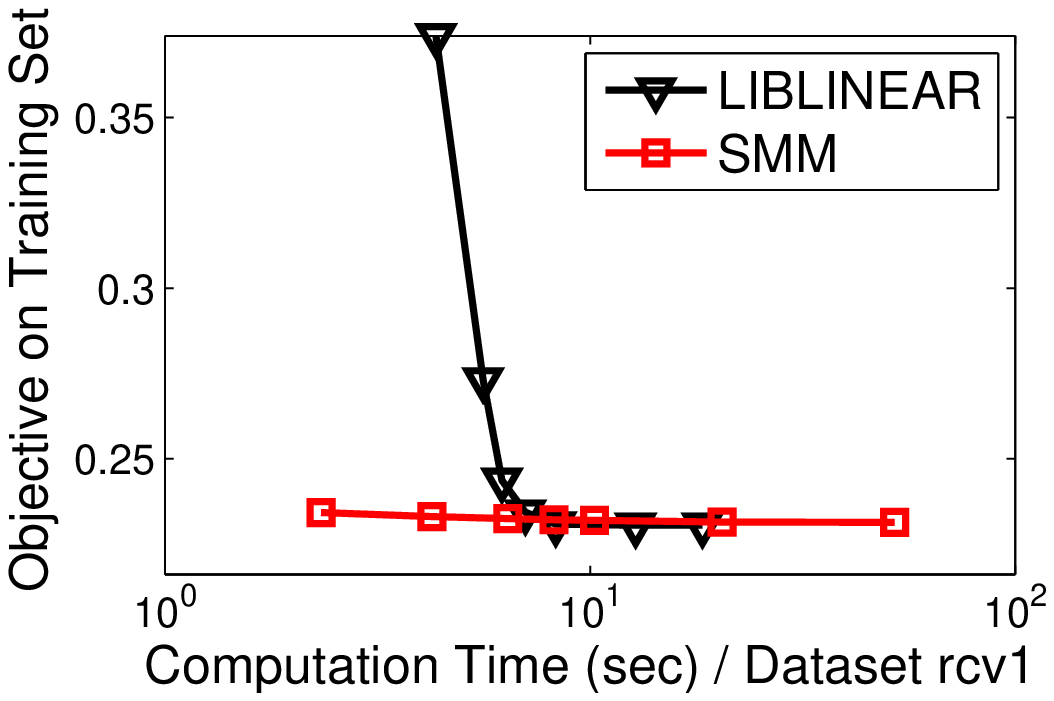}
   \includegraphics[width=0.3\linewidth]{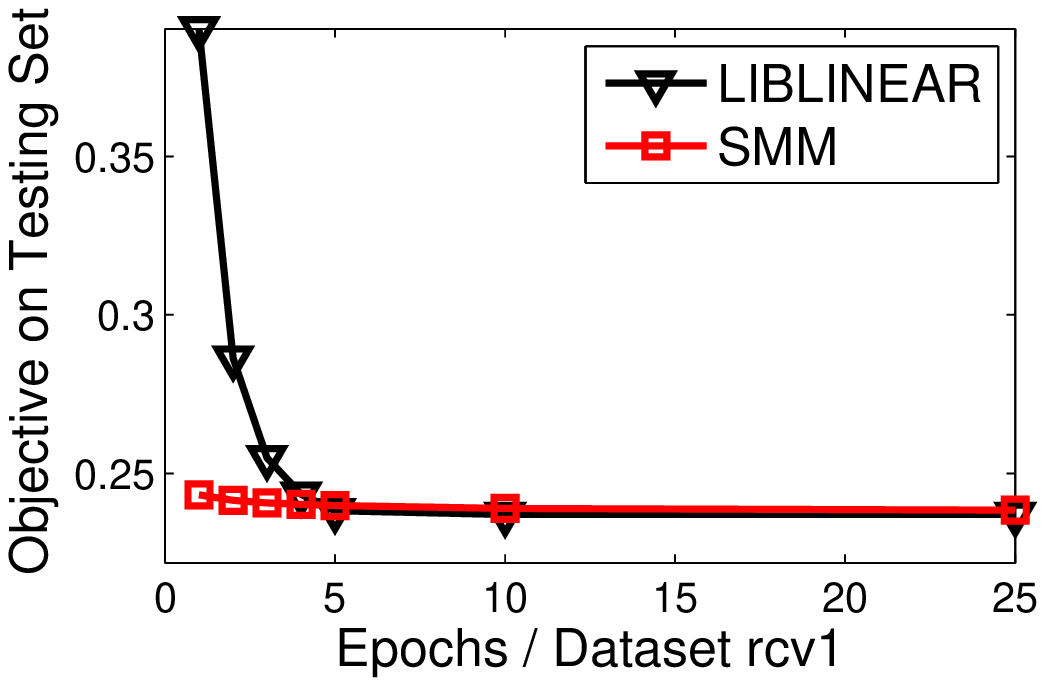}\\
   \includegraphics[width=0.3\linewidth]{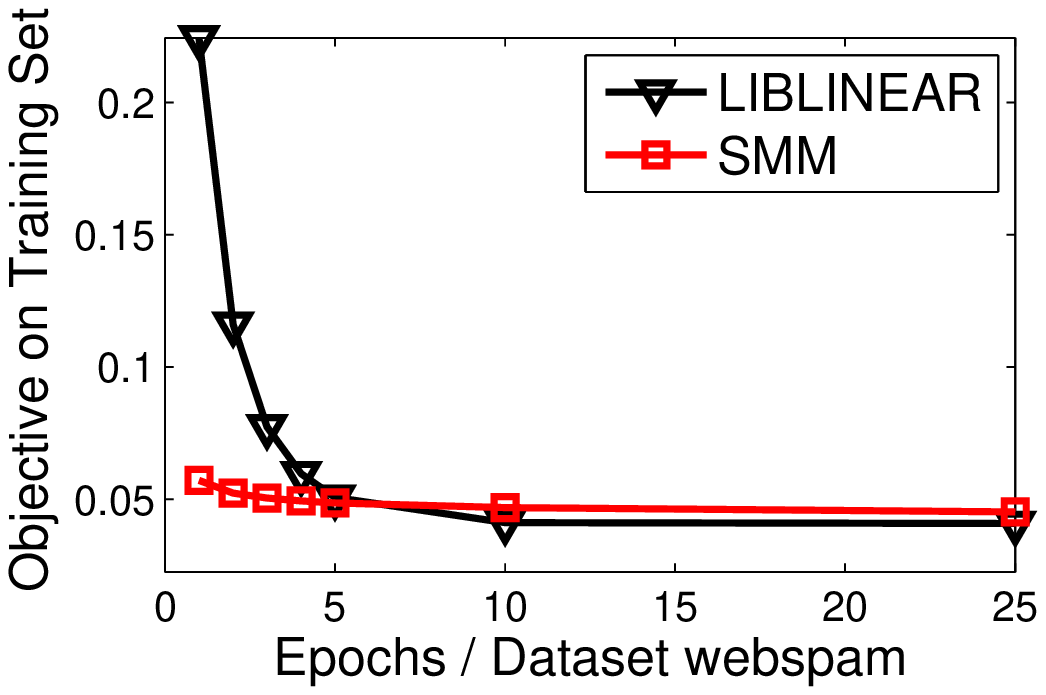}
   \includegraphics[width=0.3\linewidth]{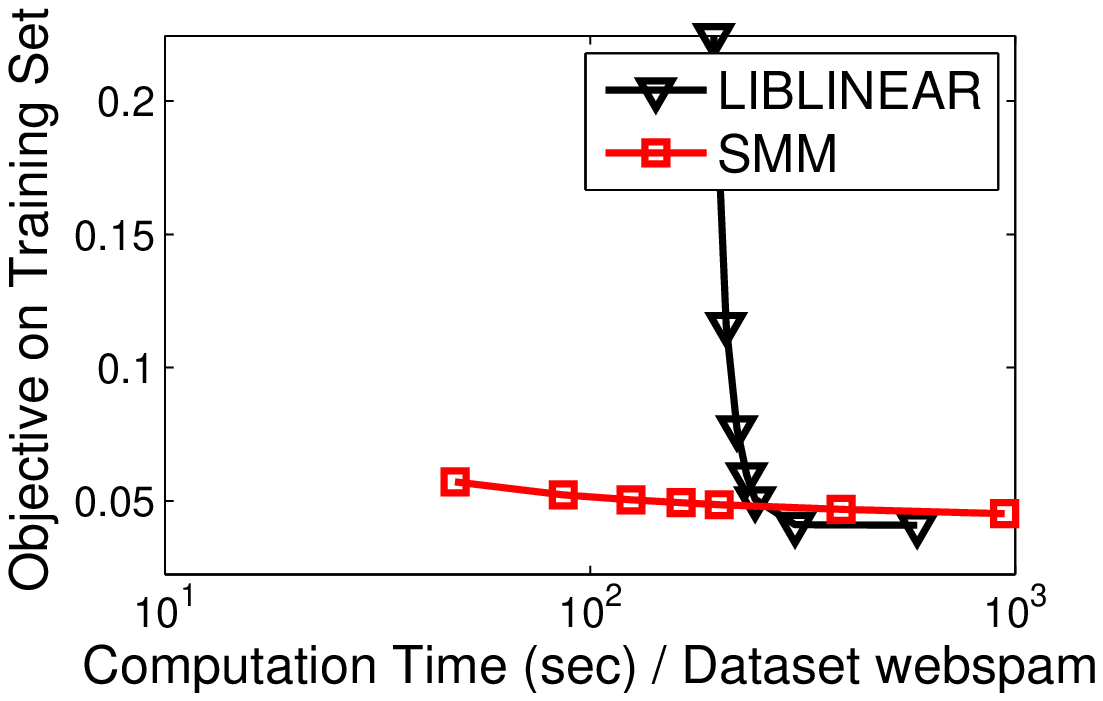}
   \includegraphics[width=0.3\linewidth]{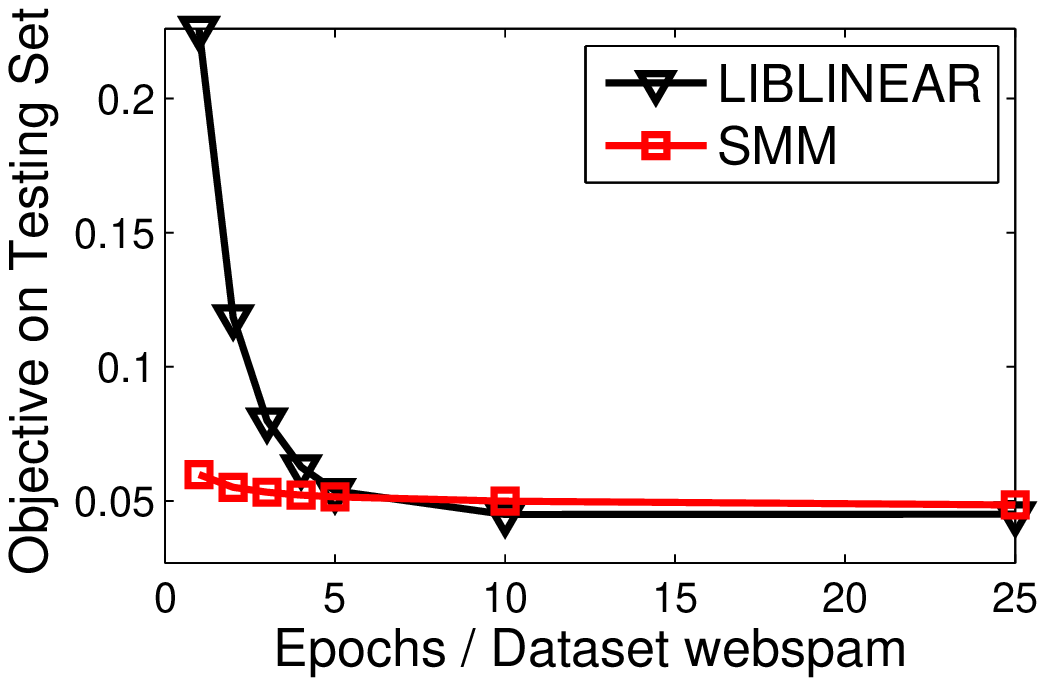}\\
   \includegraphics[width=0.3\linewidth]{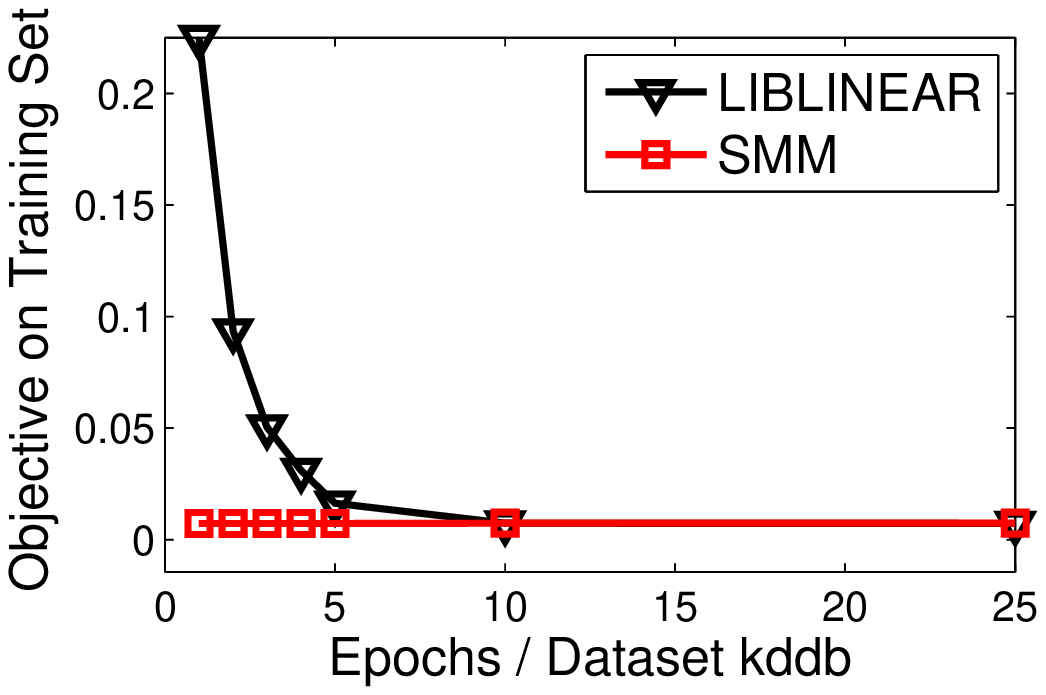}
   \includegraphics[width=0.3\linewidth]{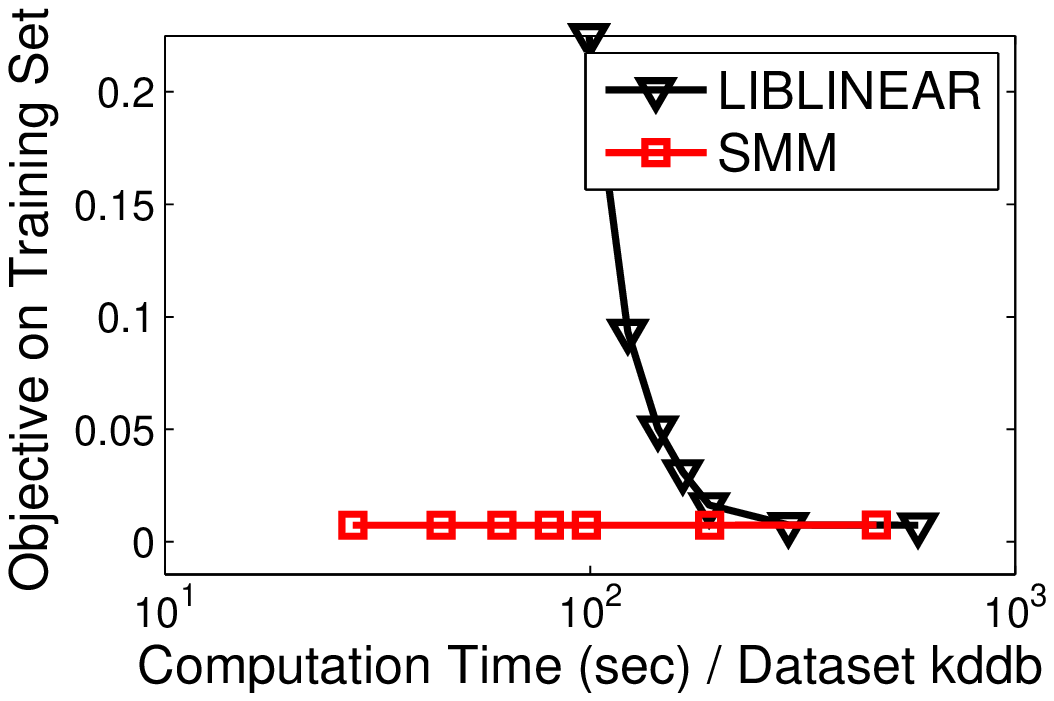}
   \includegraphics[width=0.3\linewidth]{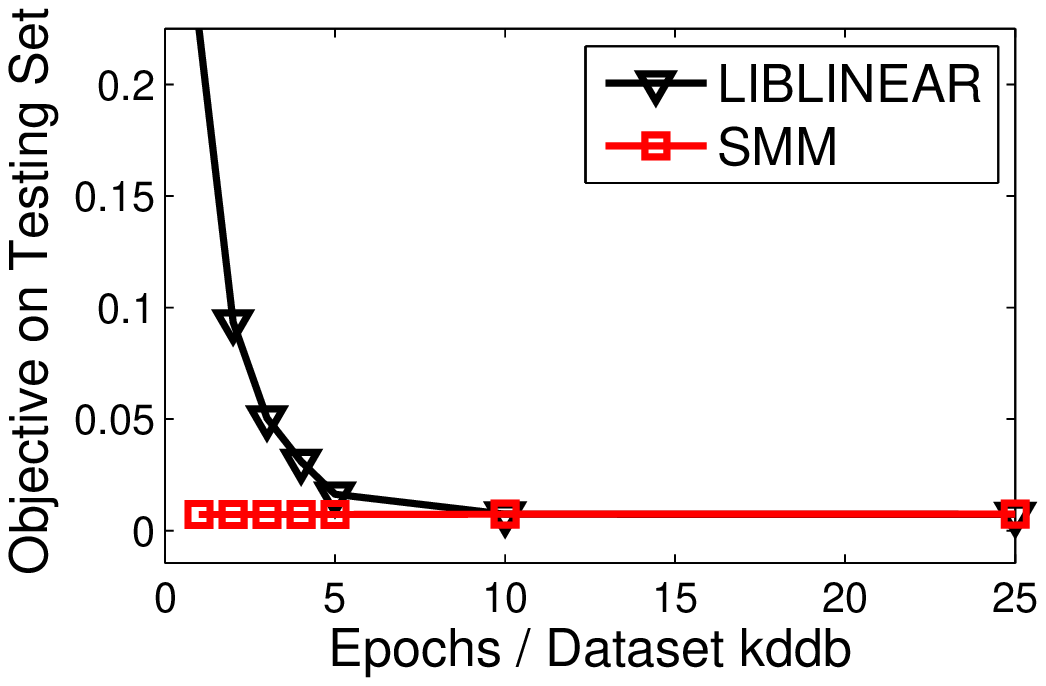}
   \vspace*{-0.2cm}
   \caption{Comparison between LIBLINEAR and SMM for the medium regularization regime. }
   \label{fig:exp_l1}
\end{figure}

%\vspace*{-0.4cm}
\subsection{Online DC Programming for Non-Convex Sparse Estimation}\label{sec:dcprog}
We now consider the same experimental setting as in the previous section, but with a
non-convex regularizer $\psi: \theta \mapsto \lambda \sum_{j=1}^p
\log(|\theta[j]|+\varepsilon)$, where $\theta[j]$ is the $j$-th entry in $\theta$. A classical way for minimizing the regularized
empirical cost $\frac{1}{N}\sum_{i=1}^N f_i(\theta) + \psi(\theta)$ is to
resort to DC programming. It consists of solving a sequence of
reweighted-$\ell_1$ problems~\cite{gasso}.  A current estimate $\theta_{n-1}$
is updated as a solution of $\min_{\theta \in \Theta}\frac{1}{N}\sum_{i=1}^N
f_i(\theta) + \lambda \sum_{j=1}^p \eta_j|\theta[j]|$, where $\eta_j \defin 1/(|\theta_{n-1}[j]|+\varepsilon)$.

In contrast to this ``batch'' methodology, we can use our framework to address
the problem online. At iteration~$n$ of Algorithm~\ref{alg:stochastic},  we
define the function~$g_n$ according to Proposition~\ref{prop:nonconvex3}:
\begin{displaymath}
   \textstyle
g_n: \theta \mapsto f_n(\theta_{n-1}) + \nabla f_n(\theta_{n-1})^\top (\theta-\theta_{n-1}) +
\frac{L}{2}\|\theta-\theta_{n-1}\|_2^2 + \lambda \sum_{j=1}^p
\frac{|\theta[j]|}{|\theta_{n-1}[j]|+\varepsilon},
\end{displaymath}
We compare our online DC programming algorithm against the batch one, and report the results
in Figure~\ref{fig:exp_log}, with $\varepsilon$ set to $0.01$.
We conclude that \emph{the batch
reweighted-$\ell_1$ algorithm always converges after $2$ or $3$
weight updates, but suffers from local minima issues. The stochastic
algorithm exhibits a slower convergence, but provides significantly better
solutions.} Whether or not there are good theoretical reasons for this fact remains
to be investigated.
Note that it would have been more rigorous to choose a bounded set~$\Theta$, which is
required by Proposition~\ref{prop:nonconvex3}. In practice, it turns not to be necessary 
for our method to work well; the iterates~$\theta_{n}$ have indeed remained
in a bounded set.
\begin{figure}[hbtp]
   \centering
   \includegraphics[width=0.245\linewidth,trim=10 0 5 0]{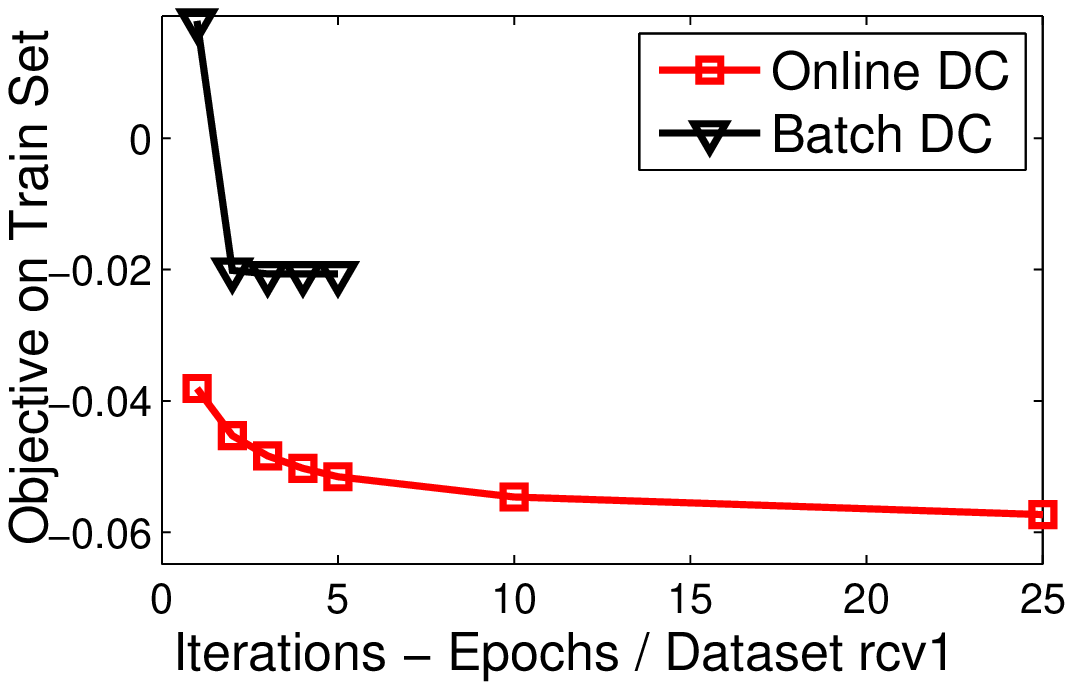}
   \includegraphics[width=0.245\linewidth,trim=10 0 5 0]{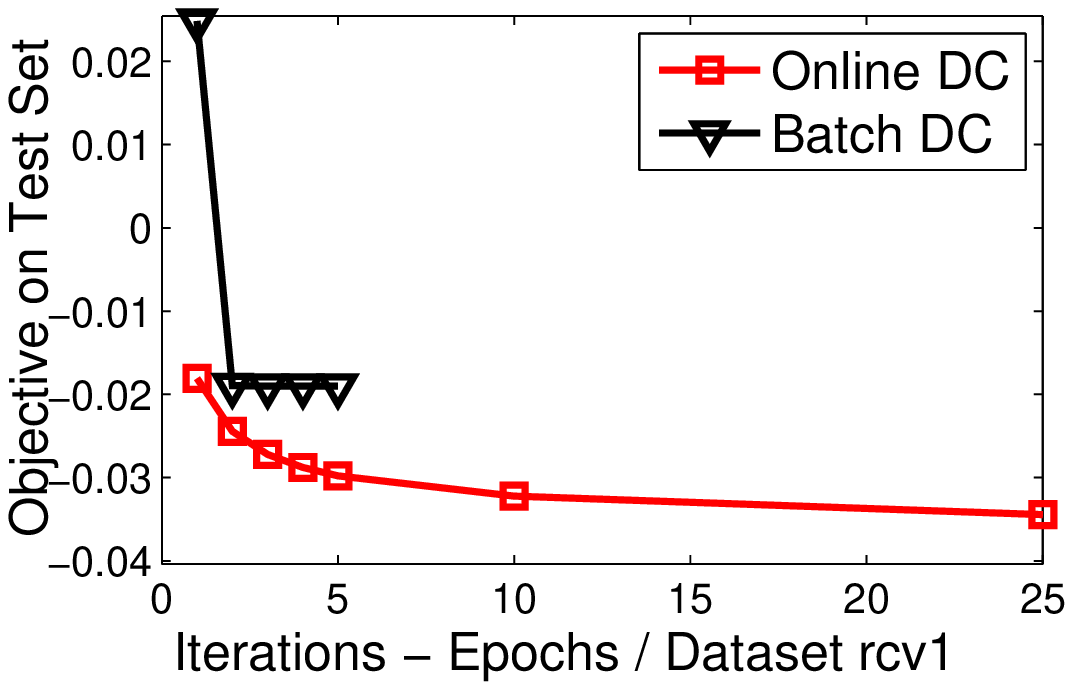}
   \includegraphics[width=0.245\linewidth,trim=10 0 5 0]{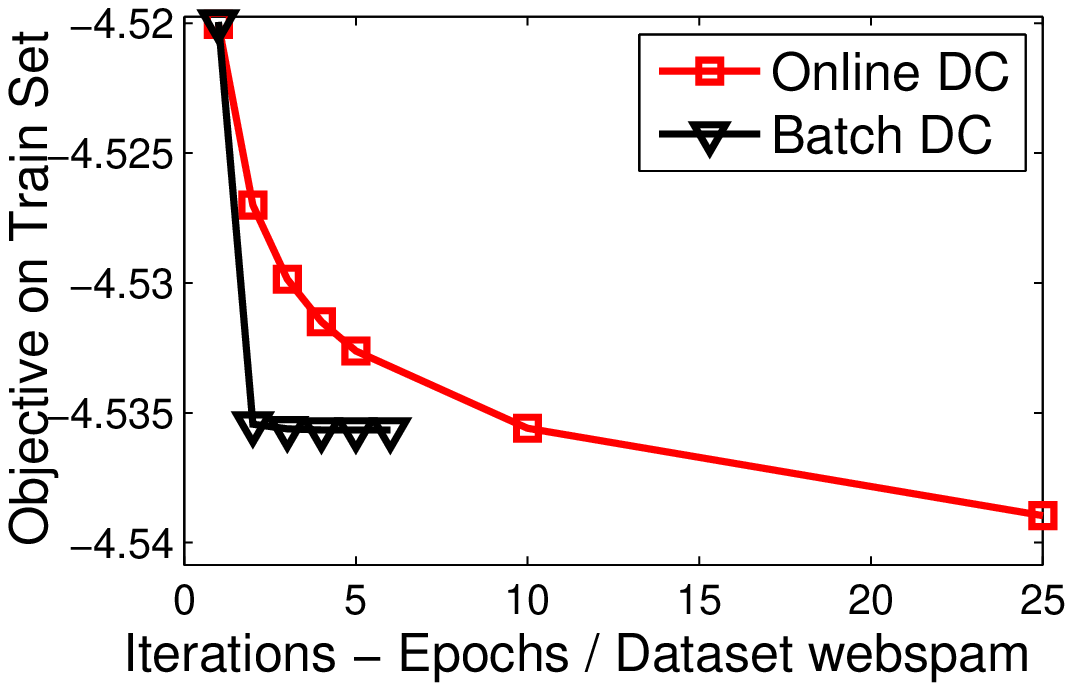}
   \includegraphics[width=0.245\linewidth,trim=10 0 5 0]{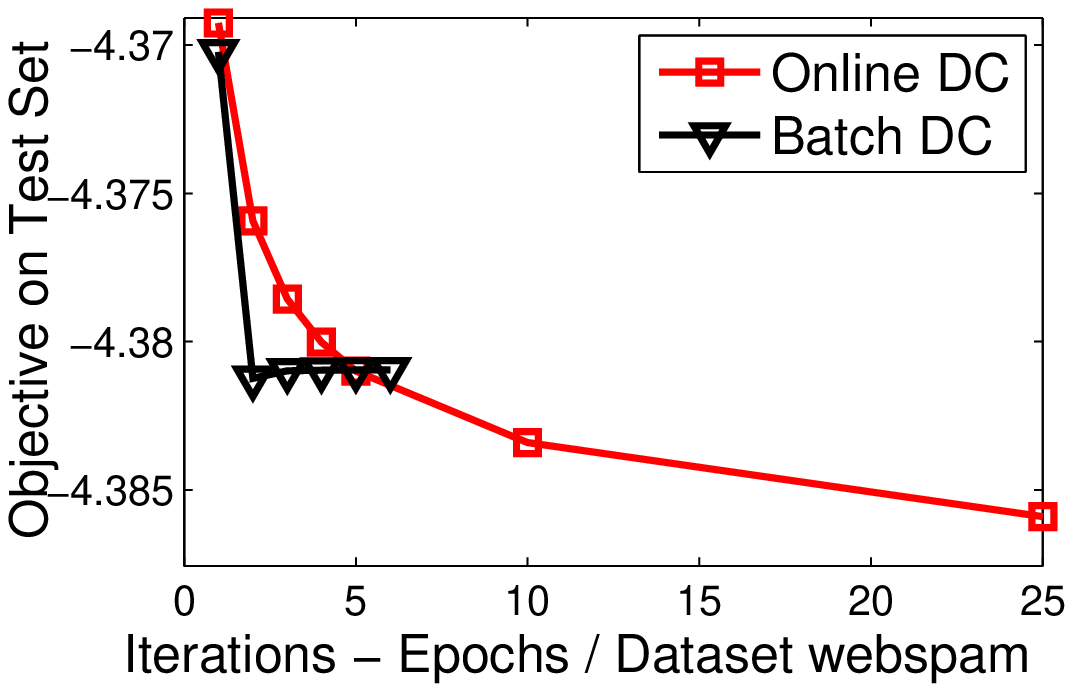}
   \vspace*{-0.6cm}
   \caption{Comparison between batch and online DC programming, with medium regularization for the datasets~\textsf{rcv1} and~\textsf{webspam}. Additional plots are provided in the supplemental material. Note that each iteration in the batch setting can perform several epochs (passes over training data).}\label{fig:exp_log}
   \vspace*{-0.3cm}
\end{figure}

\subsection{Online Structured Sparse Coding}\label{sec:sparsecoding}
In this section, we show that we can bring new functionalities to existing
matrix factorization techniques~\cite{jenatton2,mairal7}.
We are given a large collection of signals $(\x_i)_{i=1}^N$ in $\Real^m$, and
we want to find a dictionary $\D$ in $\Real^{m \times K}$ that can represent
these signals in a sparse way. The quality of $\D$ is measured through the
loss $\ell(\x,\D) \defin \min_{\alphab \in \Real^K} \frac{1}{2}
\|\x-\D\alphab\|_2^2 + \lambda_1 \|\alphab\|_1 +
\frac{\lambda_2}{2}\|\alphab\|_2^2$, where the $\ell_1$-norm can be replaced by
any convex regularizer, and the squared loss by any convex smooth loss.

Then, we are interested in minimizing the following expected cost:
\begin{displaymath}
   \min_{\D \in \Real^{m \times K}}  \E_\x\left[\ell(\x,\D) \right]  + \varphi(\D),
\end{displaymath}
where $\varphi$ is a regularizer for $\D$. In the online learning approach
of~\cite{mairal7}, the only way to regularize~$\D$ is to use a constraint set,
on which we need to be able to project efficiently; this is unfortunately not
always possible. In the matrix factorization framework of~\cite{jenatton2}, it
is argued that some applications can benefit from a structured
penalty~$\varphi$, but the approach of~\cite{jenatton2} is not easily amenable
to stochastic optimization. Our approach makes it possible by using the
proximal gradient surrogate
\begin{equation}
   g_n : \D \mapsto  \ell(\x_n,\D_{n-1}) + \trace\left(\nabla_\D \ell(\x_n,\D_{n-1})^\top (\D-\D_{n-1})\right) + \textstyle \frac{L}{2}\|\D-\D_{n-1}\|_\text{F}^2 + \varphi(\D).\label{eq:surrogate_dict}
\end{equation}
It is indeed possible to show that~$\D \mapsto \ell(\x_n,\D)$ is
differentiable, and its gradient is Lipschitz continuous with a constant~$L$
that can be explicitly computed~\cite{mairal7,mairal17}.

We now design a proof-of-concept experiment. We consider a set of $N\!=\!400\,000$
whitened natural image patches $\x_n$ of size $m\!=\!20 \times 20$ pixels. We visualize some elements from 
a dictionary~$\D$ trained by SPAMS~\cite{mairal7} on the left of Figure~\ref{fig:exp_sp1};
the dictionary elements are almost sparse, but have some residual noise among
the small coefficients. Following~\cite{jenatton2}, we propose to use a
regularization function~$\varphi$ encouraging neighbor pixels 
to be set to zero together, thus leading to a sparse structured dictionary. We consider
the collection~$\GG$ of all groups of variables corresponding to squares of $4$ neighbor pixels in~$\{1,\ldots,m\}$.
Then, we define~$\varphi(\D) \defin \gamma_1\sum_{j=1}^K
\sum_{g \in \GG} \max_{k \in g}{|\d_j[k]|} + \gamma_2\|\D\|_\text{F}^2$, where $\d_j$ is the $j$-th column
of $\D$. The penalty~$\varphi$ is a structured sparsity-inducing penalty that encourages groups of variables~$g$
to be set to zero together~\cite{jenatton2}.
Its proximal operator can be computed efficiently~\cite{mairal10}, and it is thus easy to use the
surrogates~(\ref{eq:surrogate_dict}). We set $\lambda_1\!=\!0.15$ and
$\lambda_2\!=\!0.01$; after trying a few values for~$\gamma_1$ and~$\gamma_2$ at a
reasonable computational cost, we obtain dictionaries with the desired regularization effect, as
shown in Figure~\ref{fig:exp_sp1}. Learning one dictionary of size $K\!=\! 256$
took a few minutes when performing one pass on the training data with
mini-batches of size $100$.  This experiment demonstrates that our approach is
more flexible and general than~\cite{jenatton2} and~\cite{mairal7}.
Note that it is possible to show that when $\gamma_2$ is large enough, the
iterates $\D_n$ necessarily remain in a bounded set, and thus our convergence analysis presented in Section~\ref{subsec:nonconvex}
applies to this experiment.

\begin{figure}[hbtp]
   \centering
   \includegraphics[width=0.2\linewidth]{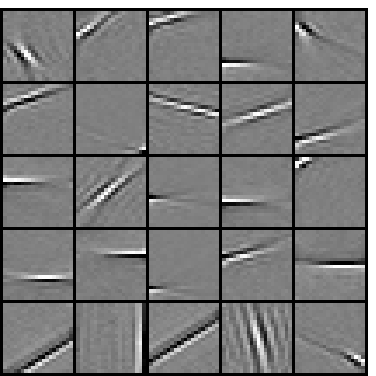}\hfill
   \includegraphics[width=0.2\linewidth]{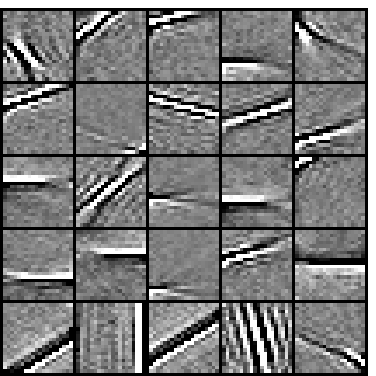} \hfill\hfill
   \includegraphics[width=0.2\linewidth]{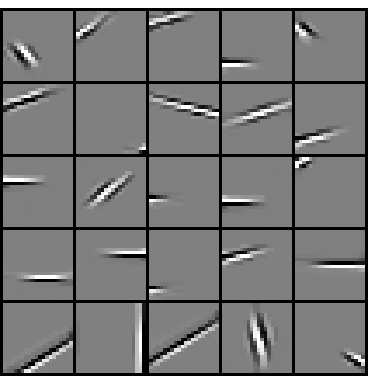}\hfill
   \includegraphics[width=0.2\linewidth]{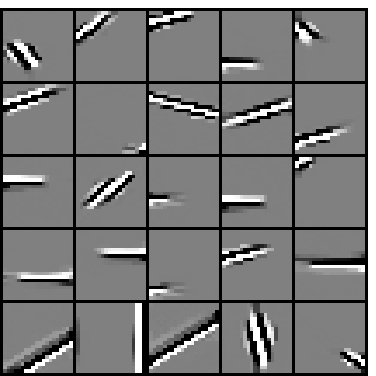}
   \vspace*{-0.2cm}
   \caption{Left: Two visualizations of $25$ elements from a larger dictionary obtained by the toolbox SPAMS~\cite{mairal7}; the second view amplifies the small coefficients. Right: the corresponding views of the dictionary elements obtained by our approach after initialization with the dictionary on the left.}
   \label{fig:exp_sp1}
\end{figure}
\vspace*{-0.3cm}

\section{Conclusion}\label{sec:ccl}
\vspace*{-0.1cm}
In this paper, we have introduced a stochastic majorization-minimization
algorithm that gracefully scales to millions of training samples. We have shown
that it has strong theoretical properties and some practical value in the
context of machine learning. We have derived from our framework several new
algorithms, which have shown to match or outperform the state of the art for
solving large-scale convex problems, and to open up new possibilities for non-convex ones. 
In the future, we would like to study surrogate functions that can exploit the
curvature of the objective function, which we believe is a crucial
issue to deal with badly conditioned datasets.

\subsubsection*{Acknowledgments}
This work was supported by
the Gargantua project (program Mastodons - CNRS).

\newpage
{\small
\bibliographystyle{plain}
\bibliography{abbrev,main}
}

\newpage
\appendix
\newcommand{\mycitep}{\citepsup}
\newcommand{\mycitet}{\citetsup}
\newcommand{\mycite}{\citesup}
\section{Mathematical Background and Useful Results}
In this paper, we use subdifferential calculus for convex functions. The
definition of subgradients and directional derivatives can be found in
classical textbooks, see, e.g.,~\cite{borwein}, \mycite{nocedal}. We denote by $\partial
f(\theta)$ the subdifferential of a convex function $f$ at a point $\theta$.
Other definitions can be found in the appendix of~\cite{mairal17}, which uses a
similar notation as ours.

In this section, we present several classical optimization and probabilistic
tools, which we use in our paper.  The first lemma is a classical quadratic
upper-bound for differentiable functions with a Lipschitz gradient. It can 
be found for instance in Lemma 1.2.3 of \mycite{nesterov4}, or in the appendix of~\cite{mairal17}.
\begin{lemma}[\bfseries Convex Surrogate for Functions with Lipschitz Gradient]~\label{lemma:upperlipschitz}\newline
Let $f: \Real^p \to \Real$ be differentiable and $\nabla f$ be $L$-Lipschitz continuous. Then, for all $\theta,\theta'$ in $\Real^p$,
   \begin{equation}
       |f(\theta') - f(\theta) - \nabla f(\theta)^\top (\theta'-\theta)| \leq \frac{L}{2}\|\theta-\theta'\|_2^2.\label{eq:lipschitz}
   \end{equation}
\end{lemma}

The next lemma is a simple relation, which will allow us to identify the subdifferential
of a convex function with the one of its surrogate at a particular point.
\begin{lemma}[\bfseries Surrogate Functions and Subdifferential]~\label{lemma:subdiff}\newline
Assume that $f,g: \Real^p \to \Real$ are convex, and that $h\defin g-f$ is
differentiable at $\theta$ in $\Real^p$ with $\nabla h(\theta)=0$. Then, $\partial f(\theta) = \partial g(\theta)$.
\end{lemma}
\begin{proof}
It is easy to show that~$g$ and $f$ have the same directional derivatives
at $\theta$ since $h$ is differentiable and $\nabla h(\theta)=0$. This is
sufficient to conclude that $\partial g (\theta)=\partial
f(\theta)$ by using Proposition 3.1.6 of \cite{borwein}, a simple lemma
relating directional derivatives and subgradients.
\end{proof}

The following lemma is a lower bound for strongly
convex functions. It can be found for instance in~\mycite{nesterov7}.
\begin{lemma}[\bfseries Lower Bound for Strongly Convex
   Functions]~\label{lemma:lower}\newline
 Let $f: \Real^p \to \Real$ be a $\mu$-strongly convex function. Let $z$ be in $\partial f(\kappa)$ for some $\kappa$ in $\Real^p$. Then, the following inequality holds for all $\theta$ in $\Real^p$:
 \begin{displaymath}
     f(\theta) \geq f(\kappa) + z^\top(\theta-\kappa) + \frac{\mu}{2}\|\theta-\kappa\|_2^2.
 \end{displaymath}
\end{lemma}
\begin{proof}
    The function $l: \theta \mapsto f(\theta) - \frac{\mu}{2}\|\theta-\kappa\|_2^2$ is
    convex by definition of strong convexity, and $l-f$ is differentiable with
    $\nabla(l-f)(\kappa)=0$.  We apply Lemma~\ref{lemma:subdiff}, which tells
    us that $z$ is in $\partial l(\kappa)$. This is sufficient to conclude, by
    noticing that a convex function is always above its tangents.
\end{proof}
The next lemma is also classical (see the appendix of~\cite{mairal17}). 
\begin{lemma}[\bfseries Second-Order Growth Property]~\label{lemma:second}\newline
 Let $f: \Real^p \to \Real$ be a $\mu$-strongly convex function and $\Theta \subseteq \Real^p$ be a convex set.
 Let $\theta^\star$ be the minimizer of $f$ on~$\Theta$. Then, the following condition holds for all $\theta$ in $\Theta$:
 \begin{displaymath}
     f(\theta) \geq f(\theta^\star) + \frac{\mu}{2}\|\theta-\theta^\star\|_2^2.
 \end{displaymath}
\end{lemma}

% Finally, we briefly recall that a strongly convex function can only have uniformly bounded subgradients on a compact set.
% \begin{lemma}[\bfseries Link Between Strong Convexity and Bounded Subgradients]~\label{lemma:bounded}\newline
%    Let $f: \Real^p \to \Real$ be a $\mu$-strongly convex function and $\Theta \subseteq \Real^p$ be a convex set.
%    Assume that the subgradients of~$f$ are uniformly bounded by $R$. Let us call $\theta^\star$ a minimizer of $f$ on $\Theta$. Then, for all $\theta$ in $\Theta$, we have
%       \begin{displaymath}
%           \|\theta^\star-\theta\|_2 \leq \frac{R}{\mu}.
%       \end{displaymath}
% \end{lemma}
% \begin{proof}
%    Let us call $\theta^\star$ a minimizer of $f$ on $\Theta$ and let us define $f^\star \defin f(\theta^\star)$.
%    According to Lemma~\ref{lemma:lower}, if $z$ is in $\partial f(\theta)$,
%    \begin{displaymath}
%       \begin{split}
%          f^\star & \geq f(\theta) + z^\top(\theta^\star-\theta) + \frac{\mu}{2}\|\theta^\star-\theta\|_2^2 \\
%                  & \geq  f(\theta) - R \|\theta^\star-\theta\|_2 + \frac{\mu}{2}\|\theta^\star-\theta\|_2^2  \\
%                  & \geq  f^\star - R \|\theta^\star-\theta\|_2 + \mu\|\theta^\star-\theta\|_2^2,  
%       \end{split}
%    \end{displaymath}
%    where the second inequality is obtained by using Lemma~\ref{lemma:second}. This is sufficient to conclude.
% \end{proof}

We now introduce a sequence of probabilistic tools, which we use in
our convergence analysis for non-convex functions.
The first one is a classical theorem on quasi-martingales, which was used
in~\cite{bottou2} for proving the convergence of the stochastic
gradient descent algorithm.
\begin{theorem}[\bfseries Convergence of Quasi-Martingales.]\label{theo:martingales}~\newline
    This presentation follows \cite{bottou2} and Proposition 9.5 and Theorem
    9.4 of~\mycite{metivier}. The original theorem is due to~\mycite{fisk}. Let $(\FF_n)_{n \geq 0}$ be an increasing family of $\sigma$-fields.
    Let $(X_n)_{n\geq 0}$ be a real stochastic process such that every random
    variable $X_n$ is bounded below by a constant independent of $n$, and
    $\FF_n$-measurable.
    Let 
    \begin{displaymath} 
       \delta_n \defin \left\{ 
       \begin{array}{ll}
          1 & ~~\text{if}~~ \E[X_{n+1}-X_n | {\FF}_n]>0, \\
          0 & ~~\text{otherwise.}
       \end{array}
       \right.
    \end{displaymath}
    If the series $\sum_{n=0}^\infty
    \E[\delta_n(X_{n+1}-X_n)]$ converges, then $(X_n)_{n \geq 0}$ is a quasi-martingale and
    converges almost surely to an integrable random variable $X_\infty$. Moreover,
    \begin{displaymath}
       \sum_{n=0}^\infty \E\big[|\E[X_{n+1}-X_n | {\FF}_n]|\big] < \infty.
    \end{displaymath}
 \end{theorem}

The next lemma is simple, but useful to prove the convergence of deterministic algorithms.
\begin{lemma}{\bf Deterministic Lemma on Non-negative Converging Series.\\ } \label{lemma:converg}
   Let $(a_n)_{n \geq 1}$, $(b_n)_{n \geq 1}$ be two non-negative
   real sequences such that the series $\sum_{n=1}^\infty a_n$ diverges, the series $\sum_{n=1}^\infty a_n b_n$ converges, and there exists $K > 0$ such that $|b_{n+1}-b_n| \leq K a_n$.
   Then, the sequence $(b_n)_{n \geq 1}$ converges to 0.
\end{lemma}
\begin{proof}
    The proof is inspired by the one of Proposition 1.2.4 of \mycite{bertsekas}.
    Since the series $\sum_{n \geq 1} a_n$ diverges, we necessarily have $\lim
    \inf_{n \to +\infty} b_n = 0$. Otherwise, it would be easy to contradict
    the assumption $\sum_{n \geq 1} a_n b_n < +\infty$.
    
    Let us now proceed by contradiction and assume that $\lim \sup_{n \to +\infty} b_n = \lambda > 0$.
    We can then build two sequences of indices $(m_j)_{j \geq 1}$ and $(n_j)_{j \geq 1}$ such that
   \begin{itemize}
       \item $m_j < n_j < m_{j+1}$,
       \item $\frac{\lambda}{3} < b_k$, for $m_j \leq k < n_j$,
       \item $b_k \leq \frac{\lambda}{3}$, for $n_j \leq k < m_{j+1}$.
    \end{itemize}
    Let $\varepsilon = \frac{\lambda^2}{9K}$ and ${\tildej}$ be large enough such that
    \begin{displaymath}
      \sum_{n=m_{\tildej}}^\infty a_n b_n < \varepsilon.
    \end{displaymath}
    Then, we have for all $j \geq {\tildej}$ and all $m$ with $m_j \leq m \leq n_j-1$,
    \begin{displaymath}
       \begin{split}
          |b_{n_j}-b_{m}| &\leq \sum_{k=m}^{n_j-1} |b_{k+1}-b_k|  \leq \frac{3K}{\lambda} \sum_{k=m}^{n_j-1} a_k \frac{\lambda}{3}  \leq \frac{3K}{\lambda} \sum_{k=m}^{n_j-1} a_k b_k  \leq \frac{3K}{\lambda} \sum_{k=m}^{+\infty} a_k b_k \\
          &\leq \frac{3K\varepsilon}{\lambda} \leq \frac{\lambda}{3}.
       \end{split}
    \end{displaymath}
    Therefore, by using the triangle inequality,
    \begin{displaymath}
       b_{m} \leq b_{n_j} + \frac{\lambda}{3} \leq \frac{2\lambda}{3}.
    \end{displaymath}
    and finally, for all $ m \geq {\tildej}$, 
    \begin{displaymath}
       b_{m}  \leq \frac{2\lambda}{3},
    \end{displaymath}
    which contradicts $\lim \sup_{n \to +\infty} b_n = \lambda > 0$. 
    Therefore, $\displaystyle b_n \underset{n \to +\infty}{\longrightarrow} 0$.

\end{proof}
We now provide a stochastic version of Lemma~\ref{lemma:converg2}.
\begin{lemma}{\bf Stochastic Lemma on Non-negative Converging Series.\\ } \label{lemma:converg2}
   Let $(X_n)_{n \geq 1}$ be a sequence of non-negative measurable random
   variables on a probability space. Let also $a_n$,~$b_n$ be two non-negative
   sequences such that $\sum_{n \geq 1} a_n = +\infty$ and $\sum_{n \geq 1} a_n
   b_n < + \infty$.
   Assume that there exists a constant $C$ such that for all $n \geq 1$, $\E[ X_n ] \leq b_n$ and $|X_{n+1}-X_n| \leq Ca_n$ almost surely.
   Then $X_n$ almost surely converges to zero.
\end{lemma}
\begin{proof}
   The following series is convergent
   \begin{displaymath}
      \E\left[ \sum_{n\geq 1} a_n X_n\right] = \sum_{n\geq 1}\E\left[  a_n X_n\right] \leq \sum_{n \geq 1} a_n b_n < +\infty,
   \end{displaymath}
   where we use the fact that the random variables are non-negative to interchange the sum and the expectation.
   We thus have that $\sum_{n\geq 1} a_n X_n$ converges with probability one.
   Then, let us call $a'_n = a_n$ and $b'_n = X_n$; the conditions
   of Lemma~\ref{lemma:converg} are satisfied for $a'_n$ and $b'_n$ with probability one, and $X_n$
   almost surely converges to zero.
\end{proof}

\section{Auxiliary Lemmas}\label{appendix:auxproofs}
In this section, we present auxiliary lemmas for our convex and non-convex analyses.
We start by presenting a lemma which is useful for both of them, and which is in fact
a core component for all results presented in~\cite{mairal17}. The proof of this lemma
is simple and available in~\cite{mairal17}.
\begin{lemma}[\bfseries Basic Properties of First-Order Surrogate Functions]~\label{lemma:basic}\newline
Let~$g$ be in $\S_{L,\rho}(f,\kappa)$ for some $\kappa$ in $\Theta$. Define the approximation error function $h\defin g-f$
 and let~$\theta'$ be the minimizer of $g$ over $\Theta$. Then, for all
 $\theta$ in~$\Theta$,
\begin{itemize}
   \item $\nabla h(\kappa) = 0$;
   \item $|h(\theta)| \leq \frac{L}{2}\|\theta-\kappa\|_2^2$;
   \item $f(\theta') \leq g(\theta') \leq f(\theta) + \frac{L}{2}\|\theta-\kappa\|_2^2 - \frac{\rho}{2}\|\theta-\theta'\|_2^2$.
\end{itemize}
\end{lemma}
% \begin{proof}
%    $h$ is differentiable and minimized by $0$. Thus, $\nabla h(\kappa)=0$.
% see~\cite{mairal17} for the rest of the proof.
% \end{proof}

\subsection{Convex Analysis}
We introduce, for all $n \geq 0$, the quantity $\xi_n \defin
\frac{1}{2}\E[\|\theta^\star-\theta_n\|_2^2]$, where $\theta^\star$ is a
minimizer of $f$ on~$\Theta$.  Our analysis also involves several quantities
that are defined recursively for all $n\geq 1$:
   \begin{equation}
   \left\{
      \begin{array}{rcl}
         A_n & \defin & (1-w_{n}) A_{n-1} + w_{n} \xi_{n-1} \\
         B_n & \defin & (1-w_{n}) B_{n-1} + w_{n} \E[f(\theta_{n-1})] \\
         C_n & \defin & (1-w_{n}) C_{n-1} + \frac{(R w_{n})^2}{2\rho}  \\
   \barg_{n} & \defin & (1-w_{n})\barg_{n-1} + w_{n} g_{n}  \\
   \barf_{n} & \defin & (1-w_{n})\barf_{n-1} + w_{n} f_{n}  
         \end{array}
       \right., \label{eq:Bn}
   \end{equation}
   where $A_0 \defin \frac{1}{L}(\rho \xi_0- f^\star)$, $B_0 \defin 0$, $C_0 \defin 0$, $\barg_0 = \barf_0 \defin \theta \mapsto \frac{\rho}{2}\|\theta-\theta_0\|_2^2$.
   Note that $\barg_0$ is $\rho$-strongly convex, and is minimized
   by~$\theta_0$. The choice for $A_0,B_0,C_0$ is driven by technical reasons,
   which appear in the proof of Lemma~\ref{lemma:surrogates}, a stochastic
   version of Lemma~\ref{lemma:basic}.

   Note that we also assume here that all the expectations above are well
   defined and finite-valued. In other words, we do not deal with measurability
   or integrability issues for simplicity, as often done in the literature~\cite{nemirovski}.

\begin{lemma}[\bfseries Auxiliary Lemma for Convex Analysis]~\label{lemma:aux}\newline
   When the functions $f_n$ are convex, and the surrogates $g_n$ are in $\S_{L,\rho}(f_n,\theta_{n-1})$, we have under assumption~\textbf{(A)} that for all $n \geq 1$,
   \begin{displaymath}
      \barg_n(\theta_{n-1}) \leq \barg_n(\theta_{n}) + \frac{ (R w_n)^2}{2\rho}.
   \end{displaymath}
\end{lemma}
\begin{proof}
First, we remark that the subdifferentials of $g_n$ and $f_n$ at
$\theta_{n-1}$ coincide by applying Lemma~\ref{lemma:subdiff}.
Then, we choose $z_n$ in $\partial g_n(\theta_{n-1}) = \partial f_n(\theta_{n-1})$, which is bounded by $R$ according to assumption~\textbf{(A)}, and we have
\begin{displaymath}
   \begin{split}
      \barg_n(\theta_n) & = (1-w_n)\barg_{n-1}(\theta_n) + w_n g_n(\theta_n) \\
                        & \geq (1\!-\!w_n)\left(\barg_{n-1}(\theta_{n-1}) \!+\! \frac{\rho}{2} \|\theta_n\!-\!\theta_{n-1}\|_2^2 \right) \!+\! w_n \left( g_n(\theta_{n-1}) \!+\! z_n^\top (\theta_n\!-\!\theta_{n-1}) \!+\!  \frac{\rho}{2} \|\theta_n\!-\!\theta_{n-1}\|_2^2\right) \\
                        & = \barg_n(\theta_{n-1}) + w_n z_n^\top (\theta_n-\theta_{n-1}) +  \frac{\rho}{2} \|\theta_n-\theta_{n-1}\|_2^2 \\
                        & \geq \barg_n(\theta_{n-1}) -R w_n \|\theta_n-\theta_{n-1}\|_2 +  \frac{\rho}{2} \|\theta_n-\theta_{n-1}\|_2^2 \\
                        & \geq \barg_n(\theta_{n-1}) - \frac{ (R w_n)^2}{2\rho}. 
   \end{split}
\end{displaymath}
The first inequality uses Lemma~\ref{lemma:second} and Lemma~\ref{lemma:lower} since $g_n$ is $\rho$-strongly convex by definition (and by induction $\barg_n$ is $\rho$-strongly convex as well); the second inequality uses Cauchy-Schwarz's inequality and the fact that the subgradients of the functions $f_n$ are bounded by $R$. 
\end{proof}

\begin{lemma}[\bfseries Another Auxiliary Lemma for Convex Analysis]~\label{lemma:stability4}\newline
   When the functions $f_n$ are convex, and the surrogates $g_n$ are in $\S_{L,\rho}(f_n,\theta_{n-1})$, we have under assumption~\textbf{(A)} that for all $n \geq 0$,
   \begin{equation}
     B_n \leq  \E[\barg_n(\theta_n)] + C_n, \label{eq:expect}
   \end{equation}
\end{lemma}
\begin{proof}
We proceed by induction, and start by showing that Eq.~(\ref{eq:expect}) is
true for $n=0$.
\begin{displaymath}
   B_0 = 0 = \E[\barg_0(\theta_0)] = \E[\barg_0(\theta_0)] + C_0.
\end{displaymath}
Let us now assume that it is true for $n-1$, and show that it is true for $n$. 
\begin{displaymath}
   \begin{split}
      B_{n} & = (1-w_n)B_{n-1} + w_n \E[f(\theta_{n-1})] \\
            & \leq (1-w_n)(\E[\barg_{n-1}(\theta_{n-1})] + C_{n-1}) + w_n \E[f(\theta_{n-1})] \\
              & = (1-w_n)(\E[\barg_{n-1}(\theta_{n-1})] + C_{n-1}) + w_n \E[f_n(\theta_{n-1})] \\
              & = (1-w_n)(\E[\barg_{n-1}(\theta_{n-1})] + C_{n-1}) + w_n \E[g_n(\theta_{n-1})] \\
              & = \E[\barg_{n}(\theta_{n-1})] + (1-w_n)C_{n-1} \\
              & \leq \E[\barg_{n}(\theta_{n})] + \frac{(R w_n)^2}{2\rho}+ (1-w_n)C_{n-1} \\
              & = \E[\barg_{n}(\theta_{n})] + C_{n}.
\end{split}
   \end{displaymath}
   The first inequality uses the induction hypothesis; the last inequality uses Lemma~\ref{lemma:aux} and the definition of~$C_n$.
We also used the fact that $\E[f_n(\theta_{n-1})] = \E[\E[f_n(\theta_{n-1})] | \FF_{n-1} ]] = \E[\E[f(\theta_{n-1})] | \FF_{n-1} ]] = \E[f(\theta_{n-1})]$,
where $\FF_{n-1}$ corresponds to the filtration induced by the past information before time $n$, such that $\theta_{n-1}$ is deterministic given~$\FF_{n-1}$.
\end{proof}

The next lemma is important; it is the stochastic version of Lemma~\ref{lemma:basic} for first-order surrogates.
\begin{lemma}[\bfseries Basic Properties of Stochastic First-Order Surrogates]~\label{lemma:surrogates}\newline
   When the functions $f_n$ are convex and the functions $g_n$ are in $\S_{L,\rho}(f_n,\theta_{n-1})$, we have under assumption~\textbf{(A)} that for all $n \geq 0$,
   \begin{displaymath}
      B_n \leq f^\star + L A_{n} - \rho \xi_n + C_n,
   \end{displaymath}
\end{lemma}
\begin{proof}
According to Lemma~\ref{lemma:stability4}, it is sufficient to show that
$\E[\barg_n(\theta_n)] \leq f^\star + L A_{n} - \rho \xi_n$ for all $n\geq 0$.
Since~$\barg_n$ is $\rho$-strongly convex and $\theta_n$ is the minimizer of $\barg_n$ over $\Theta$, we have $\E[\barg_n(\theta_n)]
\leq \E[\barg_n(\theta^\star)] - \rho \xi_n$, by using
Lemma~\ref{lemma:second}.  Thus, it is in fact sufficient to show that
$\E[\barg_n(\theta^\star)] \leq f^\star + LA_n$. For $n=0$, this inequality holds
since $\E[\barg_0(\theta^\star)] = \rho \xi_0 = f^\star + LA_0$.
We can then proceed again by induction: assume that 
$\E[\barg_{n-1}(\theta^\star)] \leq f^\star + LA_{n-1}$. Then,
\begin{displaymath}
   \begin{split}
      \E[\barg_{n}(\theta^\star)] & = (1-w_n)\E[\barg_{n-1}(\theta^\star)] + w_n \E[g_n(\theta^\star)] \\
                                  & \leq (1-w_n)(f^\star + LA_{n-1}) + w_n (\E[f_n(\theta^\star)] + L\xi_{n-1}) \\
                                  & = (1-w_n)(f^\star + LA_{n-1}) + w_n (f^\star + L\xi_{n-1}) \\
                                  & = f^\star + LA_{n},
   \end{split}
\end{displaymath}
where we have used Lemma~\ref{lemma:basic} to upper-bound the difference $\E[g_n(\theta^\star)]-\E[f_n(\theta^\star)]$ by $\xi_{n-1}$.
\end{proof}

For strongly-convex functions, we also have the following simple but useful relation between $A_n$ and~$B_n$. 
\begin{lemma}[\bfseries Relation between $A_n$ and $B_n$]~\label{lemma:AnBn}\newline
   Under assumption~{\bfseries (B)}, if $w_1=1$, we have for all $n \geq 1$,
   \begin{displaymath}
       f^\star + \mu A_n \leq B_n.
   \end{displaymath}
\end{lemma} 
\begin{proof}
   This relation is true for $n=1$ since we have $f^\star + \mu A_1 = f^\star +  \mu \xi_0 \leq
   f(\theta_0) = B_1$ by applying Lemma~\ref{lemma:second}, since $f$ is $\mu$-strongly convex according to assumption~{\bfseries (B)}.
   The rest follows by induction.
\end{proof}

\subsection{Non-convex Analysis}
When the functions $f_n$ are not convex, the convergence analysis becomes more
involved.  One key tool we use is a uniform convergence result when the
function class $\{ \x \mapsto \ell(\theta,\x) : \theta \in \Theta \}$ is ``simple enough''  in terms of entropy.
Under the assumptions made in our paper, it is indeed possible to use some results
from empirical processes~\cite{vaart}, which provides us the following lemma.
\begin{lemma}[\bfseries Uniform Convergence]~\label{lemma:uniform1}\newline
   Under assumptions~{\bfseries (A)},  {\bfseries (C)}, and {\bfseries (D)}, we have the following uniform law of large numbers:
       \begin{equation}
          \E\left[  \sup_{ \theta \in \Theta} \left| \frac{1}{n} \sum_{i=1}^n f_i(\theta) - f(\theta)     \right| \right] \leq \frac{C}{\sqrt{n}}, \label{eq:uniform}
       \end{equation}
      where $C$ is a constant, and  $\sup_{ \theta \in \Theta} \left| \frac{1}{n} \sum_{i=1}^n f_i(\theta) - f(\theta)     \right|$ converges almost surely to zero.
\end{lemma}
\begin{proof}
   We simply refer to Lemma 19.36 and Example 19.7 of~\cite{vaart}, where
   assumptions~{\bfseries (C)} and~{\bfseries (D)} ensure uniform boundness and
   squared integrability conditions.
   Note that we assume that the quantities $\sup_{ \theta \in \Theta} \left|
   \frac{1}{n} \sum_{i=1}^n f_i(\theta) - f(\theta) \right|$ are measurable.
   This assumption does not incur a loss of generality, since measurability
   issues for empirical
   processes can be dealt with rigorously~\cite{vaart}.
\end{proof}

The next lemma shows that uniform convergence applies to the weighted empirical risk $\barf_n$, defined in Eq.~(\ref{eq:Bn}), but with a different rate.
\begin{lemma}[\bfseries Uniform Convergence for $\barf_n$]~\label{lemma:uniform}\newline
   Under assumptions~{\bfseries (A)},  {\bfseries (C)}, {\bfseries (D)}, and {\bfseries (E)}, we have for all $n \geq 1$,
   \begin{displaymath}
      \E\left[  \sup_{ \theta \in \Theta} \left|  \barf_n(\theta) - f(\theta)     \right| \right] \leq C w_n \sqrt{n},
   \end{displaymath}
   where $C$ is the same as in Lemma~\ref{lemma:uniform1},
   and  $\sup_{ \theta \in \Theta} \left|  \barf_n(\theta) - f(\theta)     \right|$ converges almost surely to zero.
\end{lemma}
\begin{proof}
   We prove the two parts of the lemma separately. As in Lemma~\ref{lemma:uniform1}, we assume all the quantities of interest to be measurable.

   \proofstep{First part of the lemma}
   Let us fix $n > 0$. It is easy to show that $\barf_n$ can be written as $\barf_n = \sum_{i=1}^n w_n^i f_i$ for some non-negative weights~$w_n^i$ with $w_n^n = w_n$.
   Let us also define the empirical cost $F_i \defin \frac{1}{n-i+1}\sum_{j=i}^n f_j$.
   According to~(\ref{eq:uniform}), we have $\E\left[\sup_{\theta \in \Theta} | F_i(\theta)-f(\theta)|\right] \leq \frac{C}{\sqrt{n-i+1}}$.
   We now remark that
   \begin{displaymath}
      \barf_n -f = \sum_{i=1}^n (w_n^i-w_n^{i-1}) (n-i+1) (F_i-f),
   \end{displaymath}
   where we have defined $w_n^0 \defin 0$.
   This relation can be proved by simple calculation. We obtain the first part by using the triangle inequality, and the fact that $w_n^i \geq w_n^{i-1}$ for all $i$:
   \begin{displaymath}
      \begin{split}
         \E\left[\sup_{\theta \in \Theta}|\barf_n(\theta) -f(\theta)|\right] & \leq \E\left[  \sum_{i=1}^n (w_n^i-w_n^{i-1}) (n-i+1) \sup_{\theta \in \Theta}|F_i(\theta)-f(\theta)|\right] \\
                                                                             & = \sum_{i=1}^n (w_n^i-w_n^{i-1}) (n-i+1) \E\left[\sup_{\theta \in \Theta}|F_i(\theta)-f(\theta)|\right] \\
                                                                             & \leq \sum_{i=1}^n (w_n^i-w_n^{i-1}) C \sqrt{n-i+1} \\
                                                                             & \leq \sqrt{n}C \sum_{i=1}^n (w_n^i-w_n^{i-1}) \\
                                                                             & = C\sqrt{n}w_n.
      \end{split}
   \end{displaymath}
   This is unfortunately not sufficient to show that $\E\left[\sup_{\theta \in \Theta}|\barf_n(\theta) -f(\theta)|\right]$ converges to zero almost surely. We will show this fact by using Lemma~\ref{lemma:converg2}. 

   \proofstep{Second part of the lemma}
   We call $X_n = \sup_{\theta \in \Theta}|\barf_n(\theta) -f(\theta)|$.
   We have 
   \begin{displaymath}
      \begin{split}
         X_n-X_{n-1} & = \sup_{\theta \in \Theta}|(1-w_n)(\barf_{n-1}(\theta) -f(\theta)) + w_n(f_n(\theta)-f(\theta))| - X_{n-1} \\
                     & \leq \sup_{\theta \in \Theta} w_n| f_n(\theta)-f(\theta)| - w_n X_{n-1} \leq 2 M w_n
      \end{split}
   \end{displaymath}  
   Let us denote by $\theta^\star_n$ a point in $\Theta$ such that $X_n = |\barf_n(\theta^\star_n)-f(\theta^\star_n)|$.
   We also have
   \begin{displaymath}
      \begin{split}
         X_n-X_{n-1} & = \sup_{\theta \in \Theta}|(1-w_n)(\barf_{n-1}(\theta) -f(\theta)) + w_n(f_n(\theta)-f(\theta))| - X_{n-1} \\
                     & \geq (1-w_n) X_{n-1} + w_n (f_n(\theta^\star_{n-1})-f(\theta^\star_{n-1})) - X_{n-1} \\
                     & \geq -w_n X_{n-1} + w_n (f_n(\theta^\star_{n-1})-f(\theta^\star_{n-1})) \\
                     & \geq -w_n 4M,
      \end{split}
   \end{displaymath}  
   where we use again the fact that all functions $f_n$, $\barf_n$ and $f$ are bounded by $M$.
   Thus, we have shown that $|X_n-X_{n-1}| \leq 4M w_n$. Call $a_n=w_n$ and
   $b_n=w_n \sqrt{n}$, then the conditions of Lemma~\ref{lemma:converg2} are
   satisfied, and $X_n$ converges almost surely to zero.
\end{proof}

Finally, the next lemma illustrates why the strong convexity of the surrogates is important.
\begin{lemma}[\bfseries Stability of the Estimates]~\label{lemma:stability_simple}\newline
   When $g_n$ is in $\S_{L,\rho}(f,\theta_{n-1})$,
   \begin{displaymath}
      \|\theta_n-\theta_{n-1}\|_2 \leq \frac{2Rw_n}{\rho}.
   \end{displaymath}
\end{lemma}
\begin{proof}
   Because the surrogates $g_n$ are $\rho$-strongly convex, we have from Lemma~\ref{lemma:second}
   \begin{displaymath}
      \begin{split}
         \frac{\rho}{2}\|\theta_n-\theta_{n-1}\|_2^2 & \leq \barg_n(\theta_{n-1})- \barg_n(\theta_n)  \\
                                                     & =  w_n\left( g_n(\theta_{n-1}) - g_n(\theta_n)\right) + (1-w_n)\left( \barg_{n-1}(\theta_{n-1}) - \barg_{n-1}(\theta_n) \right) \\
                                                     & \leq  w_n\left( g_n(\theta_{n-1}) - g_n(\theta_n)\right) \\
                                                     & \leq  w_n\left( f_n(\theta_{n-1}) - f_n(\theta_n)\right) \\
                                                     & \leq  R w_n\|\theta_n-\theta_{n-1}\|_2.
      \end{split}
   \end{displaymath}
   The second inequality comes from the fact that $\theta_{n-1}$ is a minimizer of $\barg_{n-1}$;
   the third inequality is because $g_n(\theta_{n-1})=f_n(\theta_{n-1})$ and $g_n \geq f_n$. This is sufficient to conclude.
\end{proof}

\section{Proofs of the Main Lemmas and Propositions}\label{appendix:proofs}
\printproofs

\section{Additional Experimental Results}
We present in Figures~\ref{fig:exp_l1b} and~\ref{fig:exp_l1c} some additional experimental
comparisons, which complement the ones of Section~\ref{sec:logistic}.
Figures~\ref{fig:exp_logc} and~\ref{fig:exp_logb} present additional plots from 
the experiment of Section~\ref{sec:dcprog}.
Finally, we present three dictionaries corresponding to the experiment of Section~\ref{sec:sparsecoding}
in Figures~\ref{fig:exp_sp2}, \ref{fig:exp_sp3} and \ref{fig:exp_sp4}.

\begin{figure}[hbtp]
   \centering
   \includegraphics[width=0.3253\linewidth]{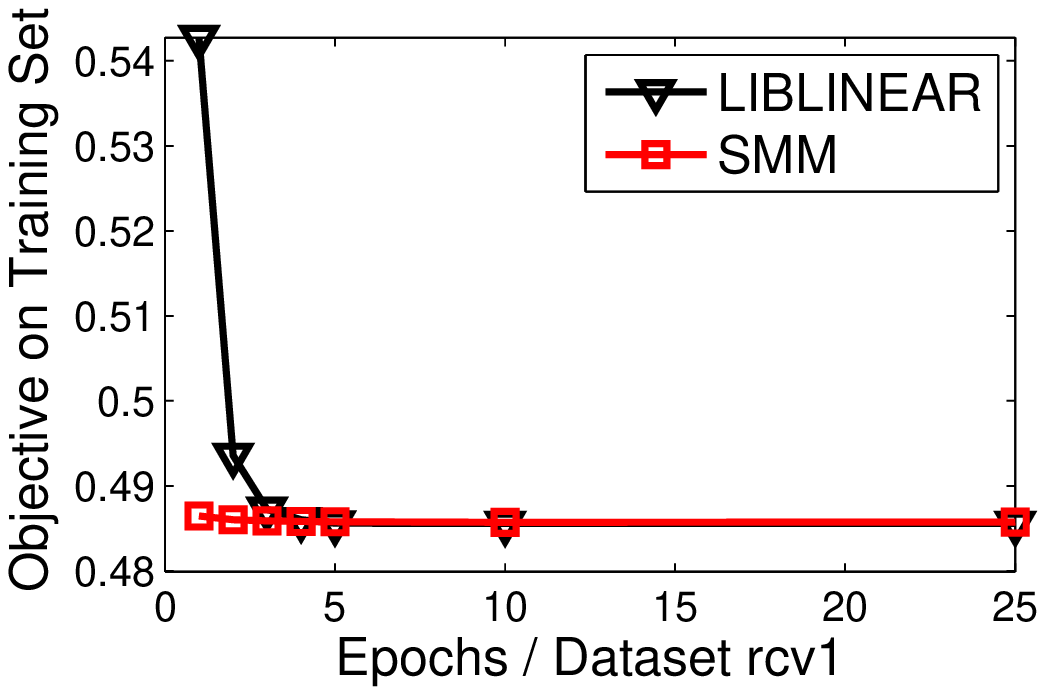}
   \includegraphics[width=0.3253\linewidth]{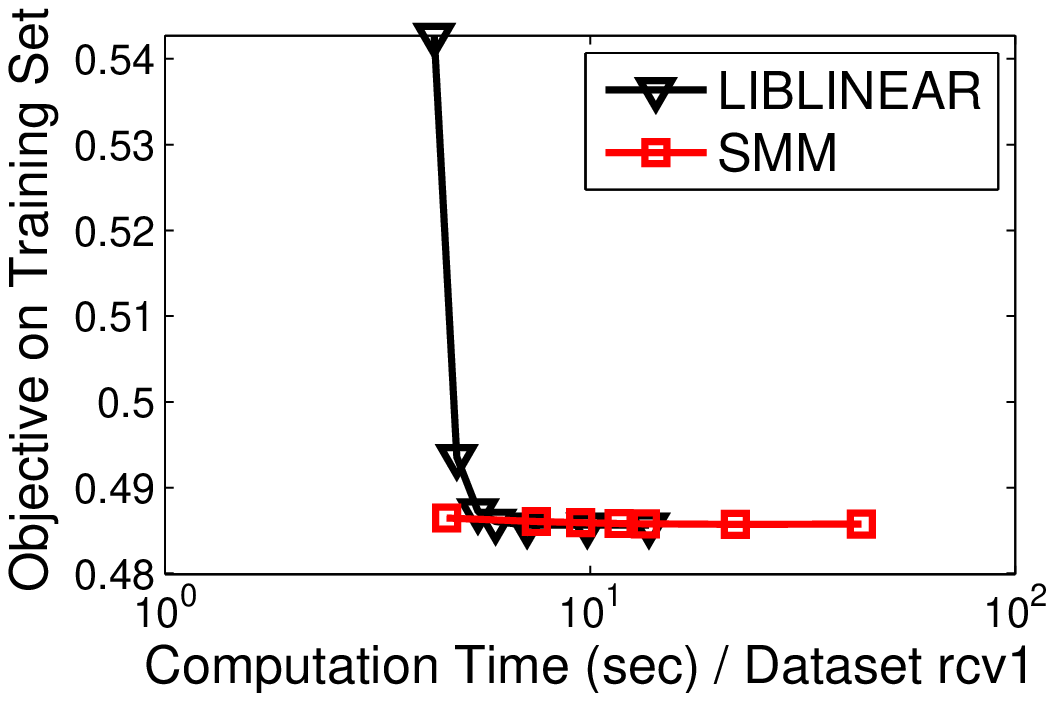}
   \includegraphics[width=0.3253\linewidth]{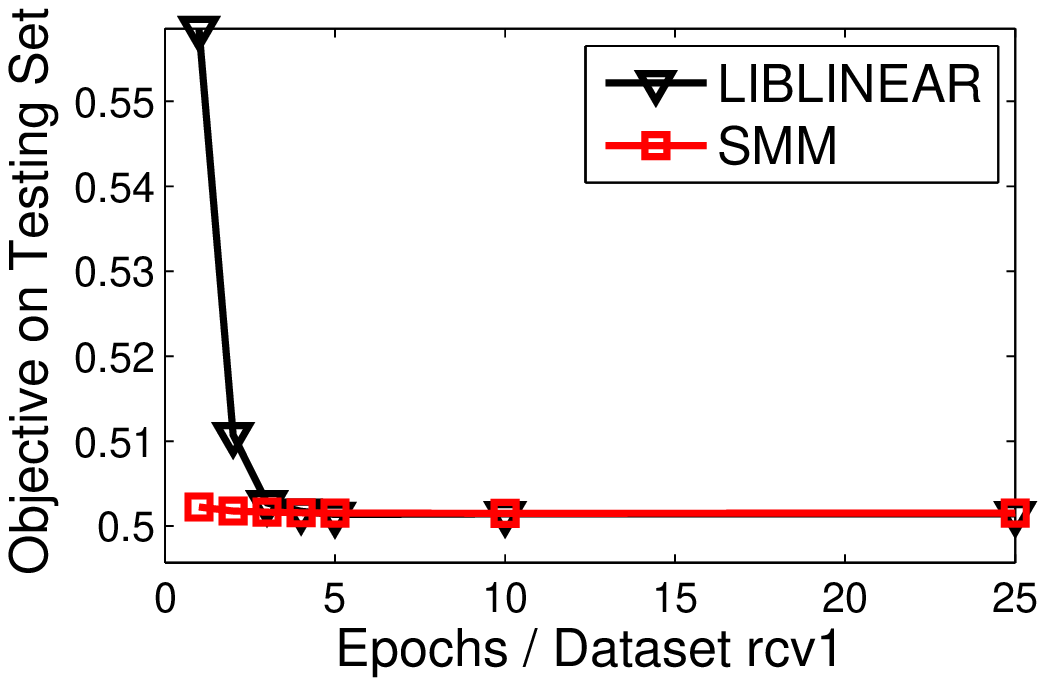}\\
   \includegraphics[width=0.3253\linewidth]{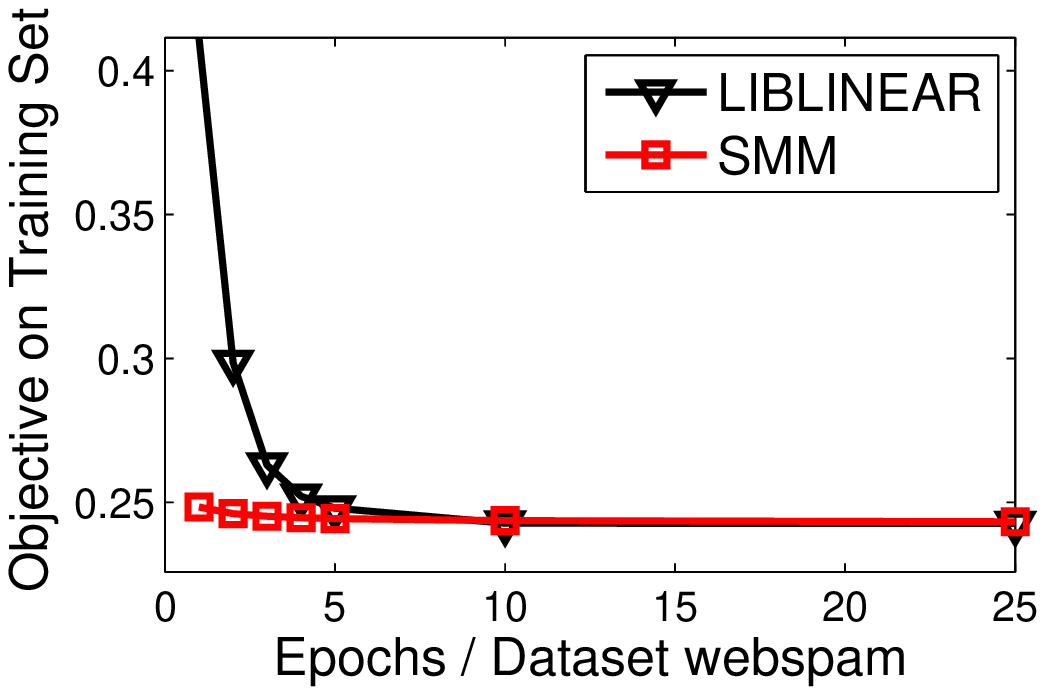}
   \includegraphics[width=0.3253\linewidth]{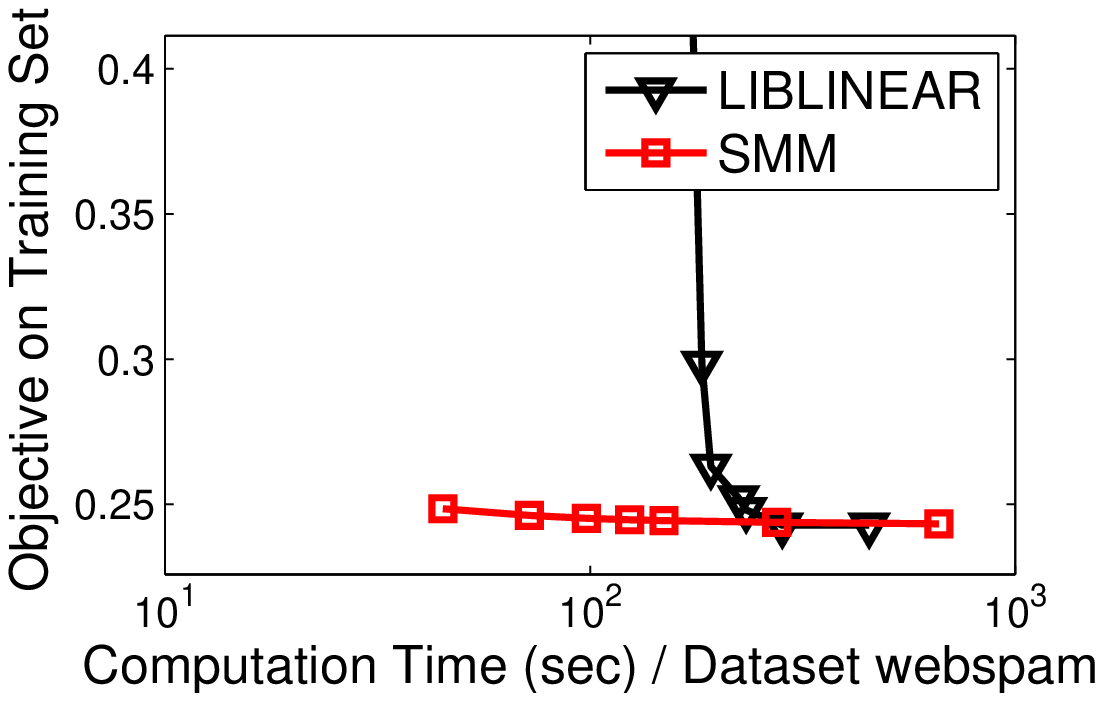}
   \includegraphics[width=0.3253\linewidth]{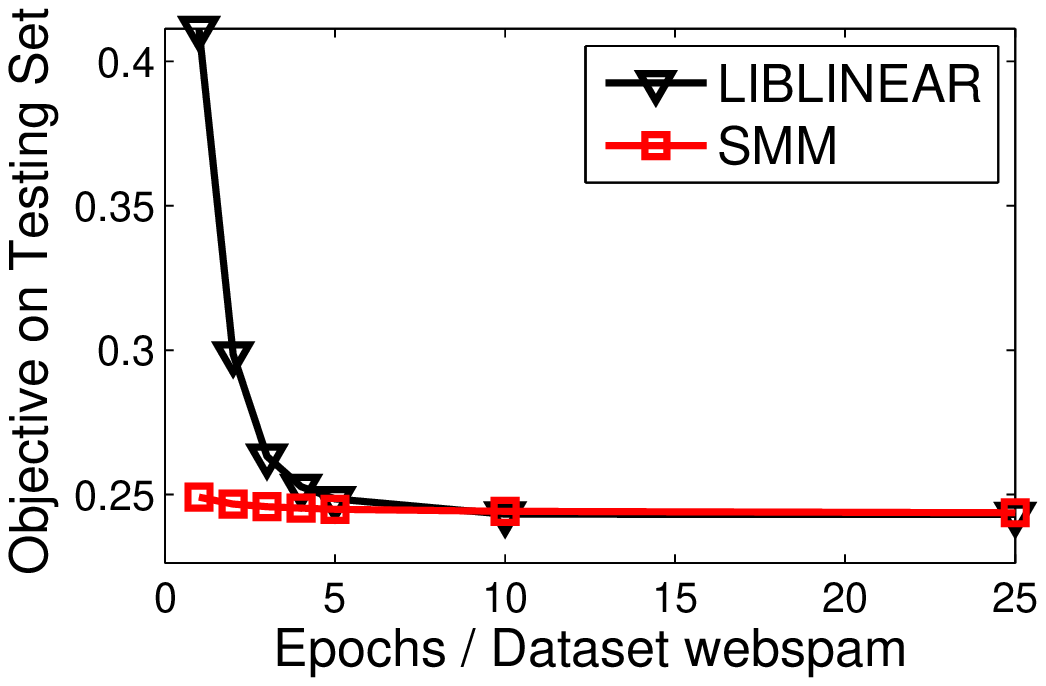}\\
   \includegraphics[width=0.3253\linewidth]{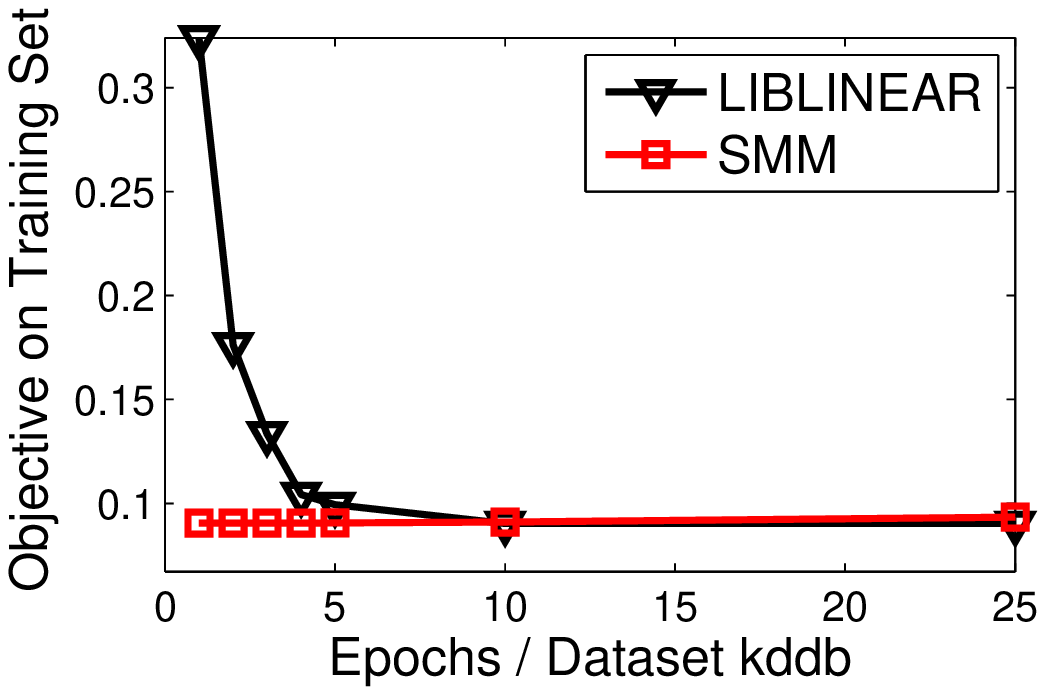}
   \includegraphics[width=0.3253\linewidth]{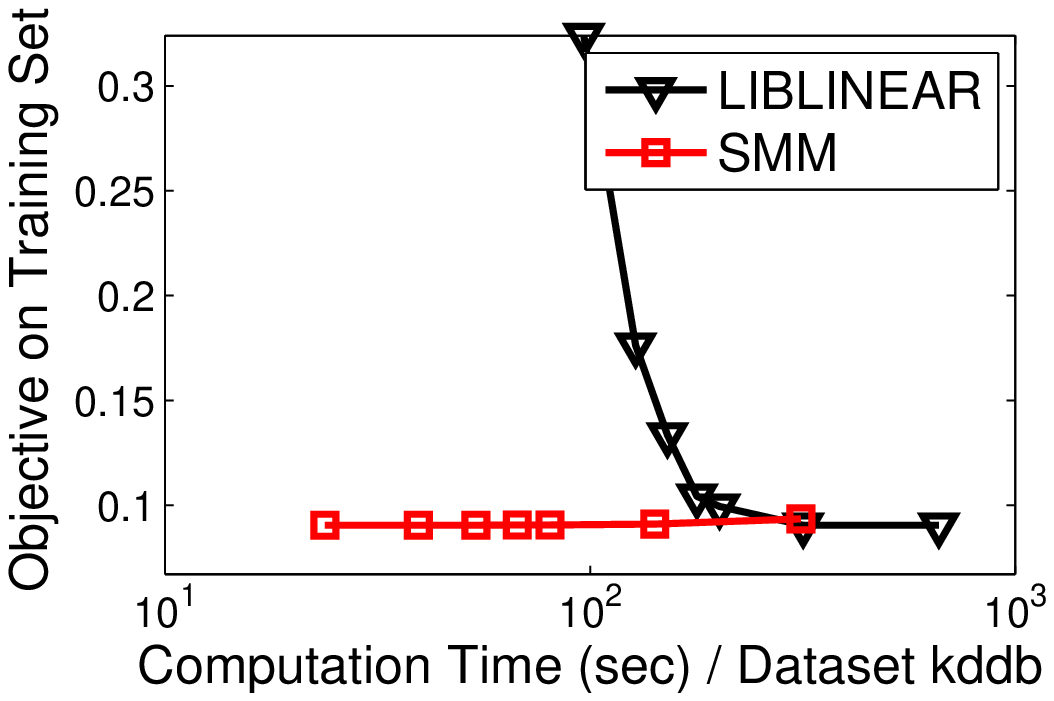}
   \includegraphics[width=0.3253\linewidth]{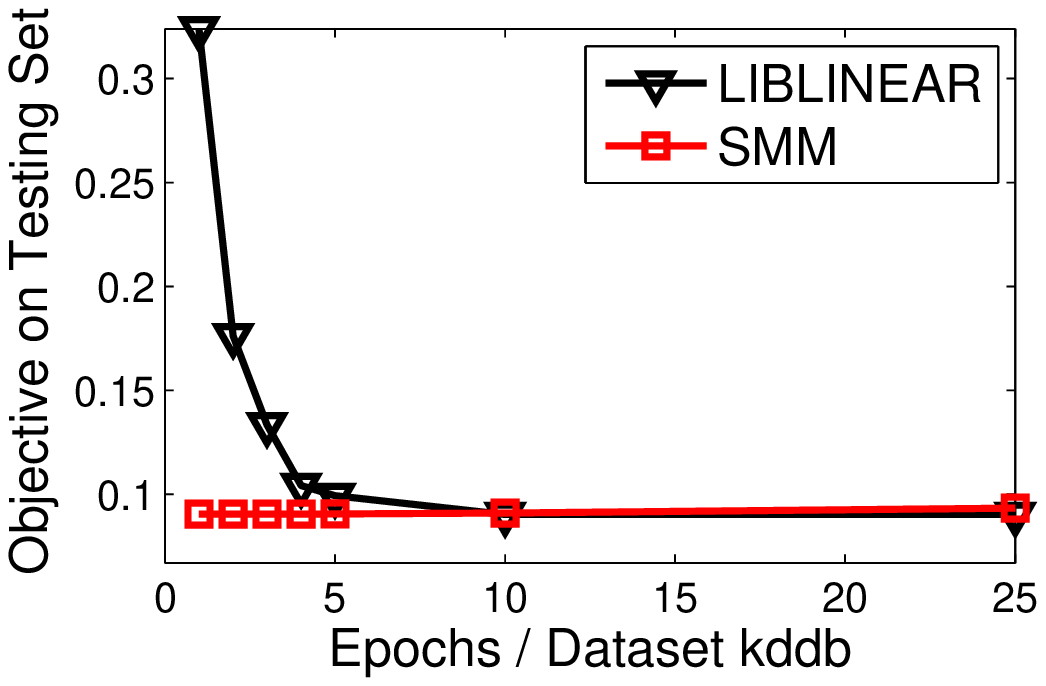}
   \caption{Comparison between LIBLINEAR and SMM in the high regularization regime for $\ell_1$-logistic regression. }
   \label{fig:exp_l1b}
\end{figure}

\begin{figure}[hbtp]
   \centering
   \includegraphics[width=0.3253\linewidth]{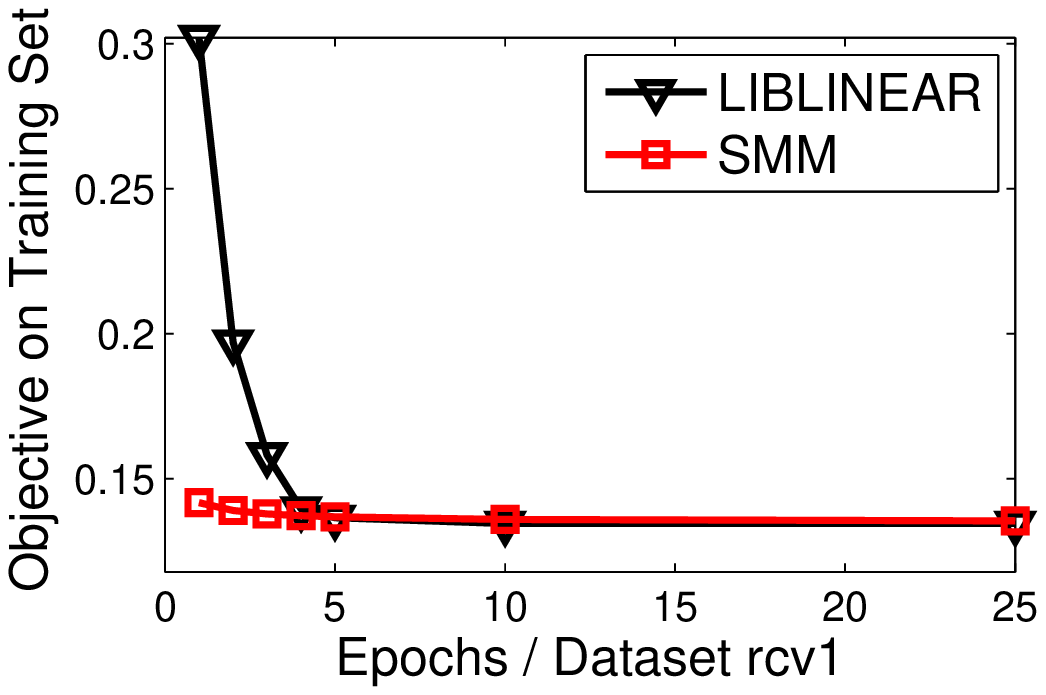}
   \includegraphics[width=0.3253\linewidth]{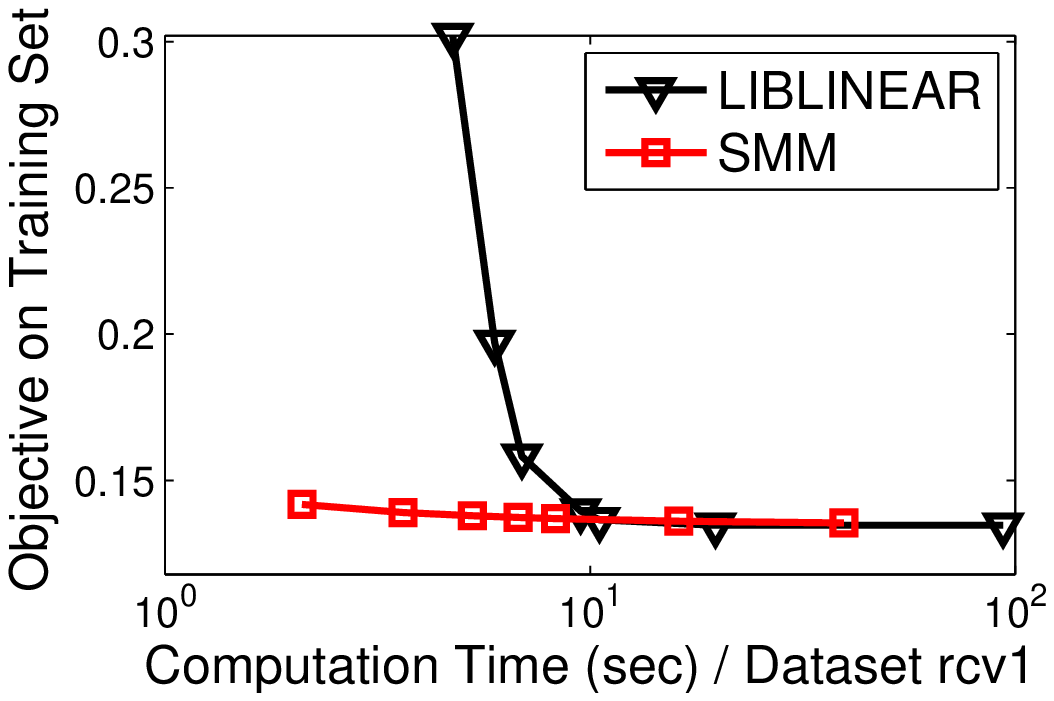}
   \includegraphics[width=0.3253\linewidth]{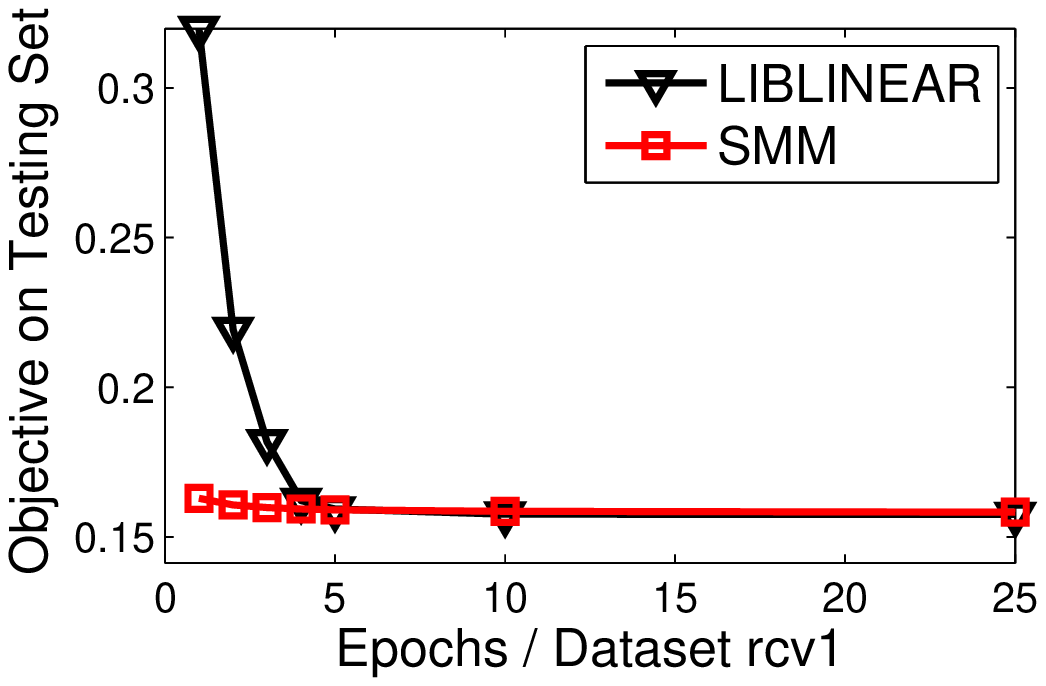}\\
   \includegraphics[width=0.3253\linewidth]{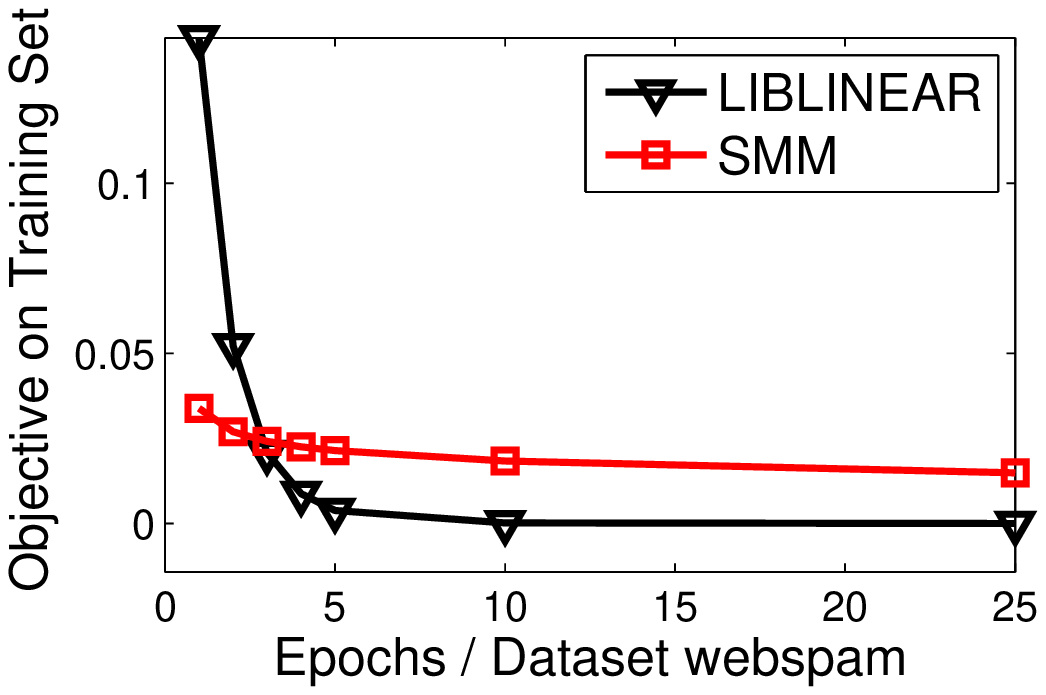}
   \includegraphics[width=0.3253\linewidth]{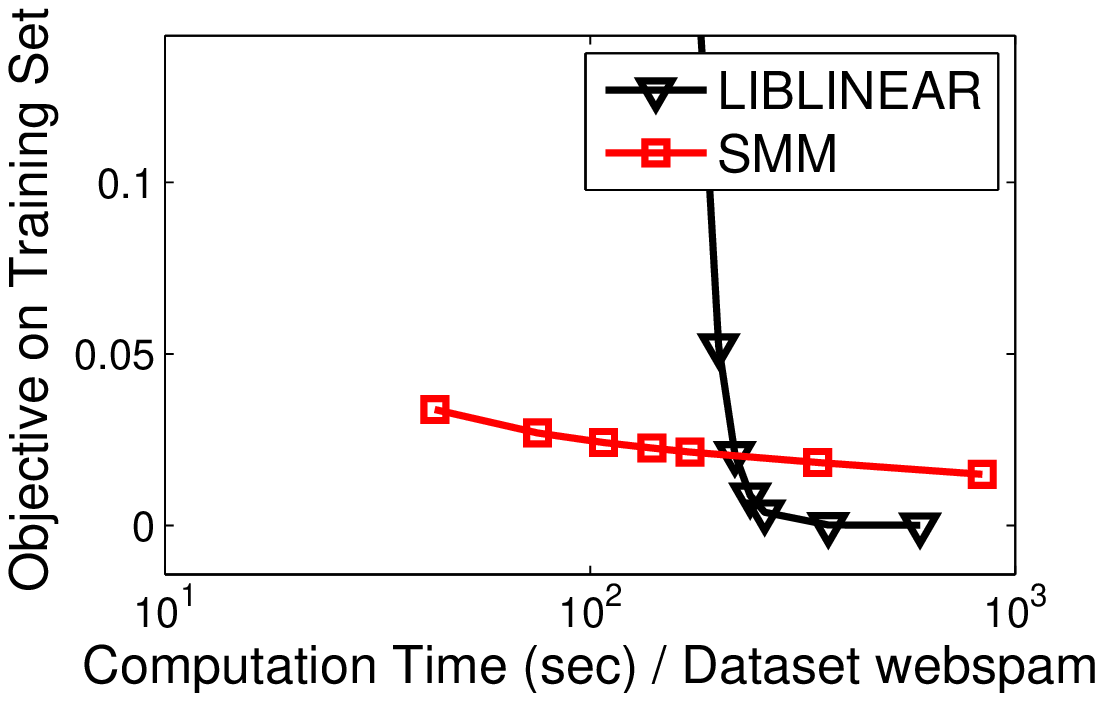}
   \includegraphics[width=0.3253\linewidth]{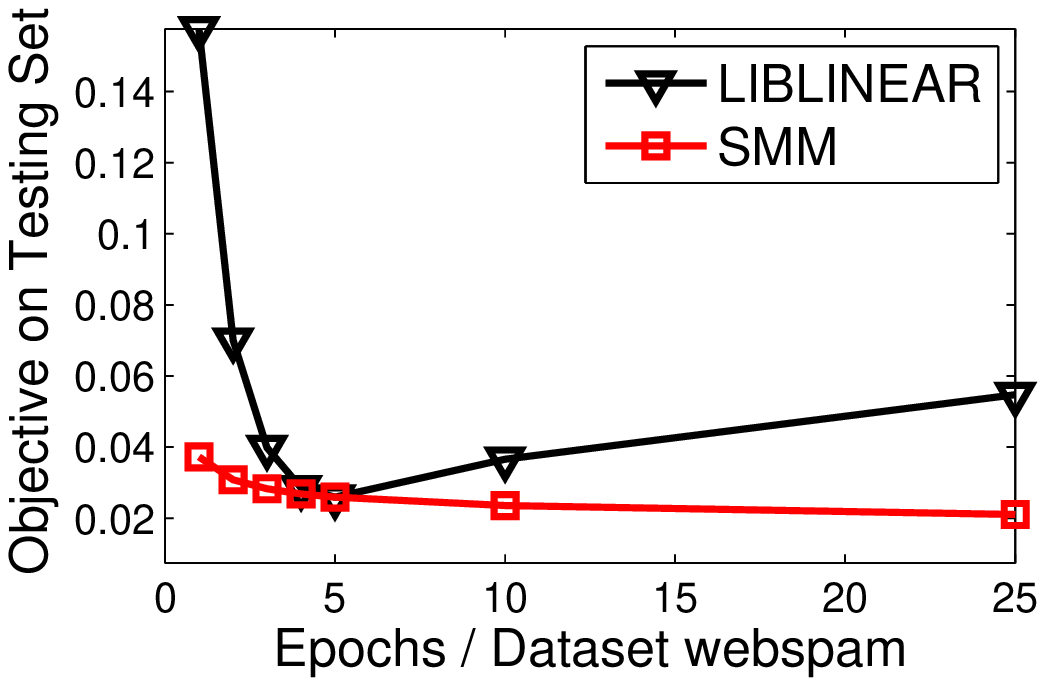}\\
   \includegraphics[width=0.3253\linewidth]{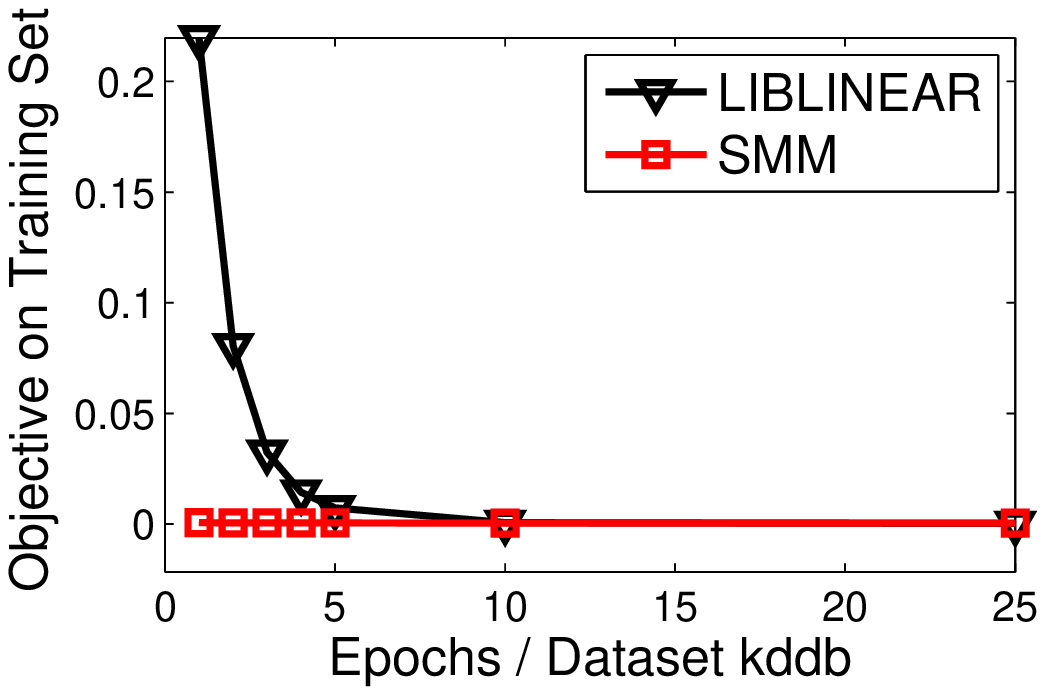}
   \includegraphics[width=0.3253\linewidth]{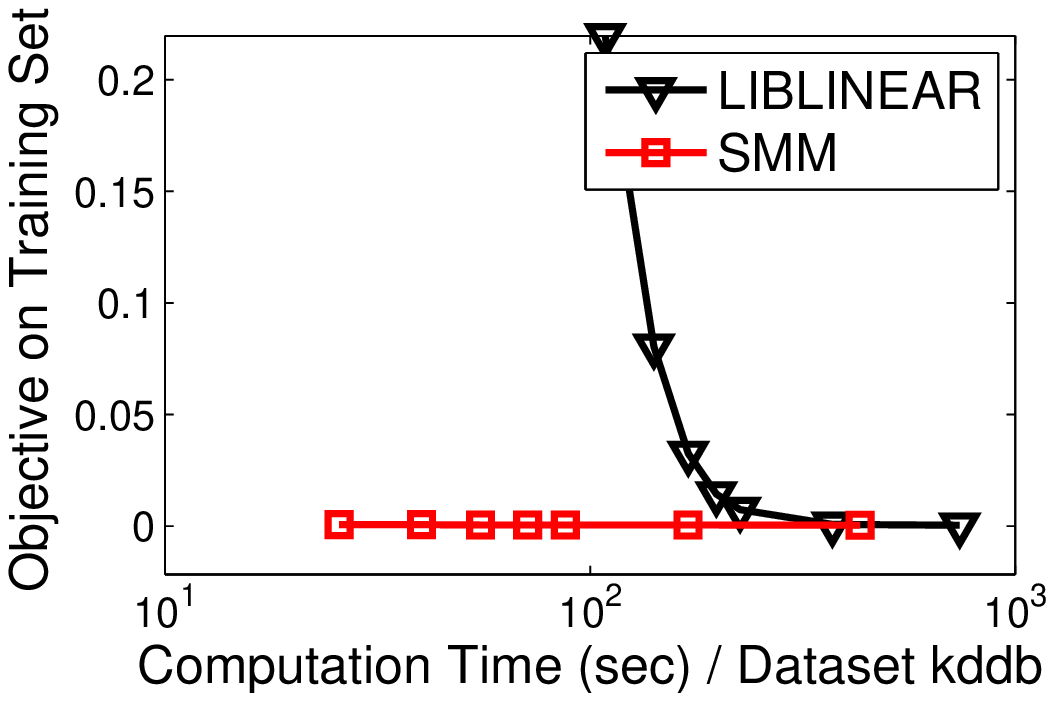}
   \includegraphics[width=0.3253\linewidth]{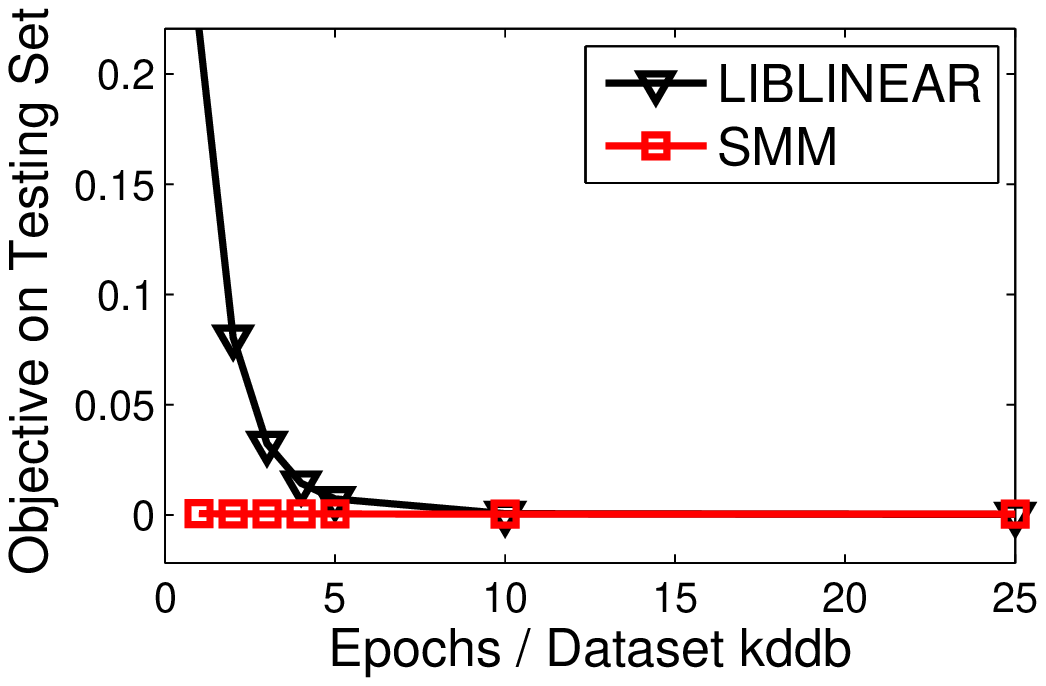}
   \caption{Comparison between LIBLINEAR and SMM in the low regularization regime for $\ell_1$-logistic regression. }
   \label{fig:exp_l1c}
\end{figure}

\begin{figure}[hbtp]
   \centering
   \includegraphics[width=0.4\linewidth]{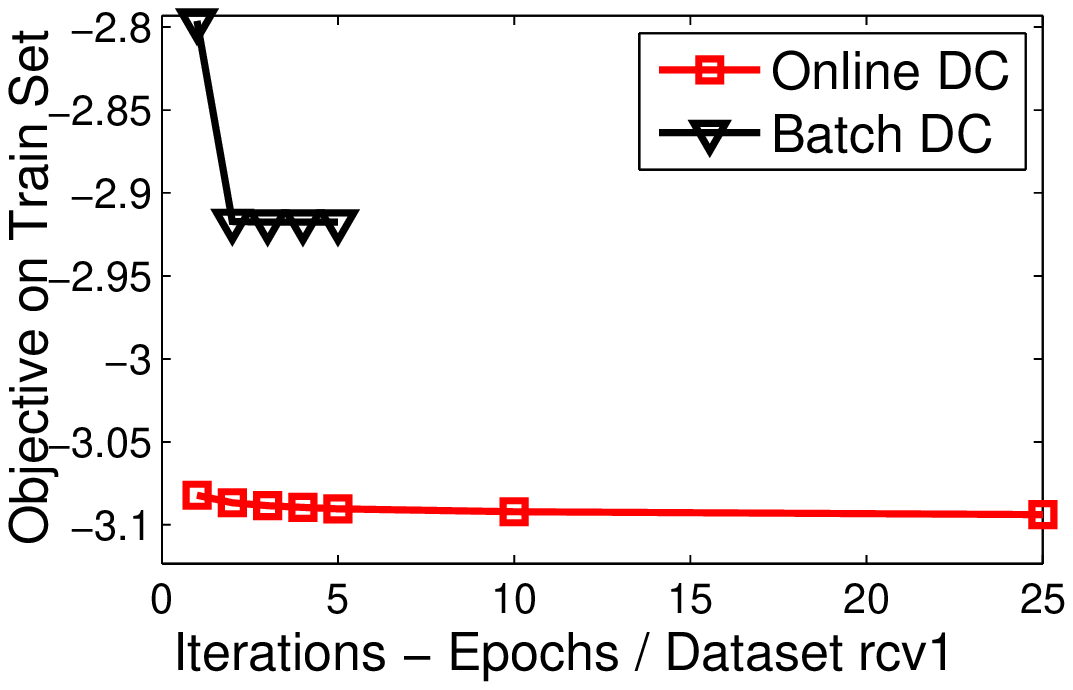}
   \includegraphics[width=0.4\linewidth]{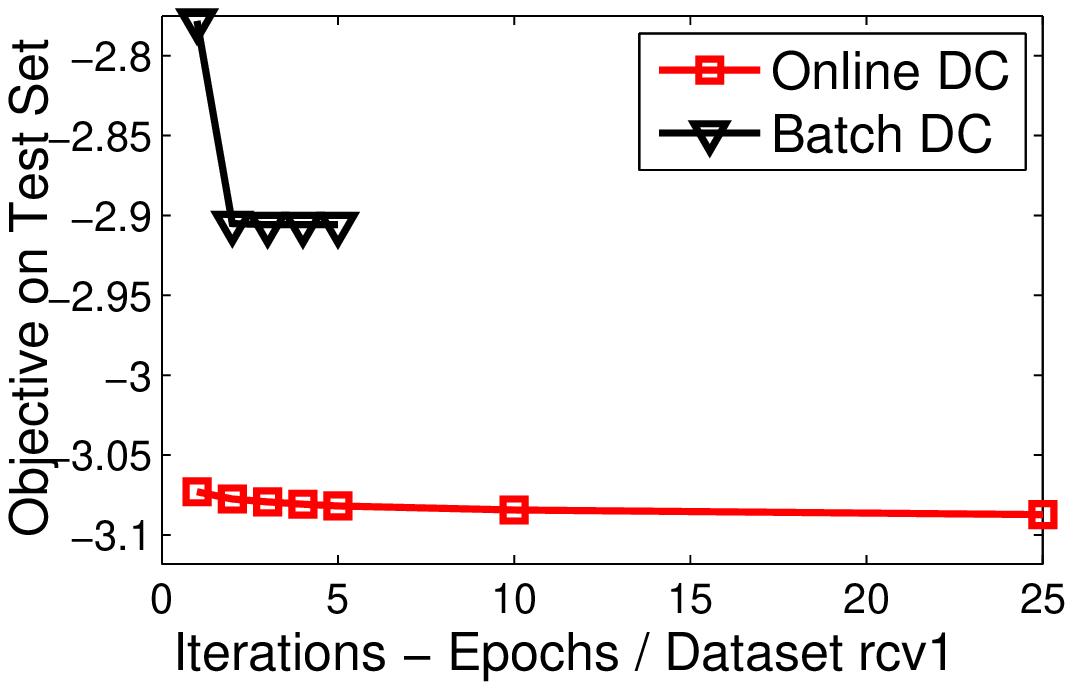} \\
   \includegraphics[width=0.4\linewidth]{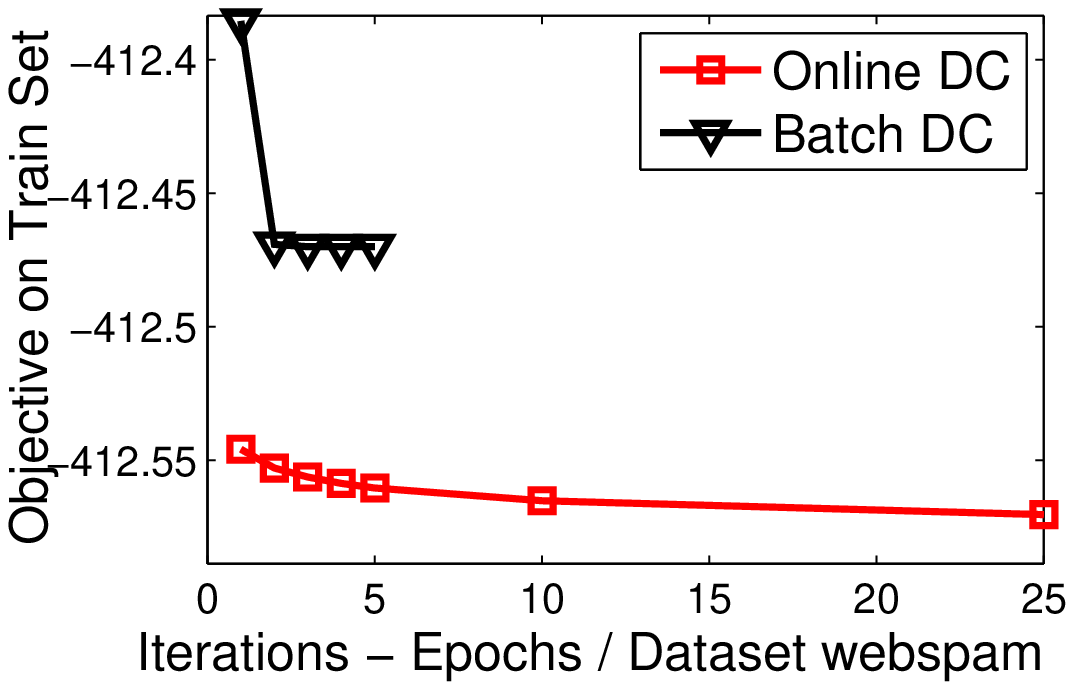}
   \includegraphics[width=0.4\linewidth]{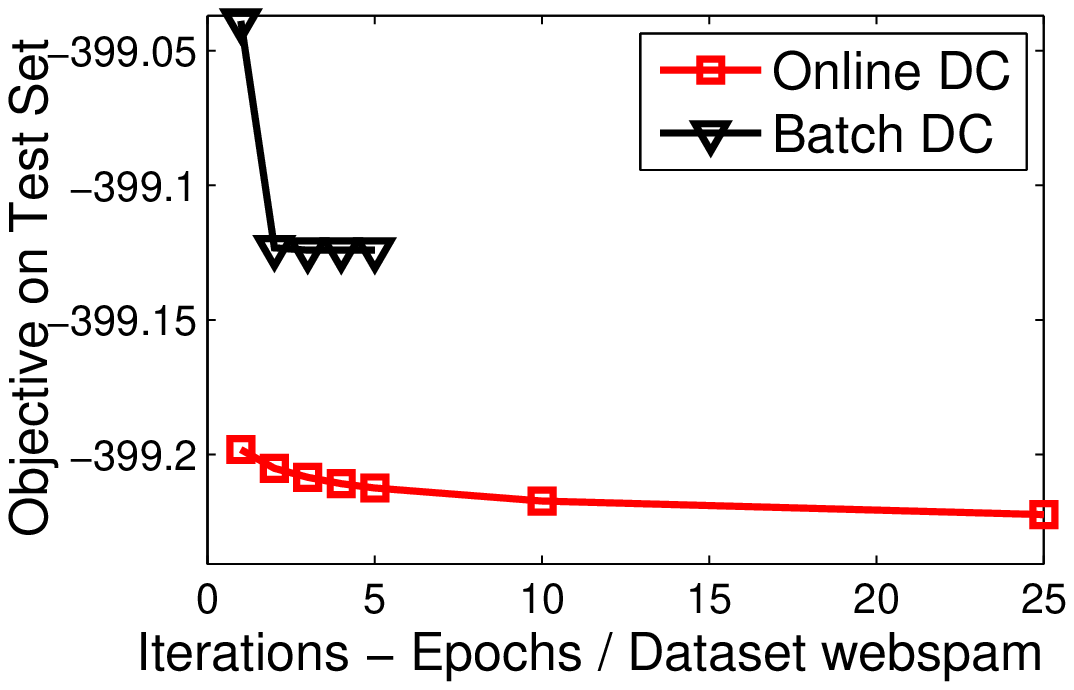}
   \caption{Comparison between batch and online DC programming, with high regularization for the datasets~\textsf{rcv1} and~\textsf{webspam}. Note that each iteration in the batch setting can perform several epochs.}\label{fig:exp_logc}
\end{figure}

\begin{figure}[hbtp]
   \centering
   \includegraphics[width=0.4\linewidth]{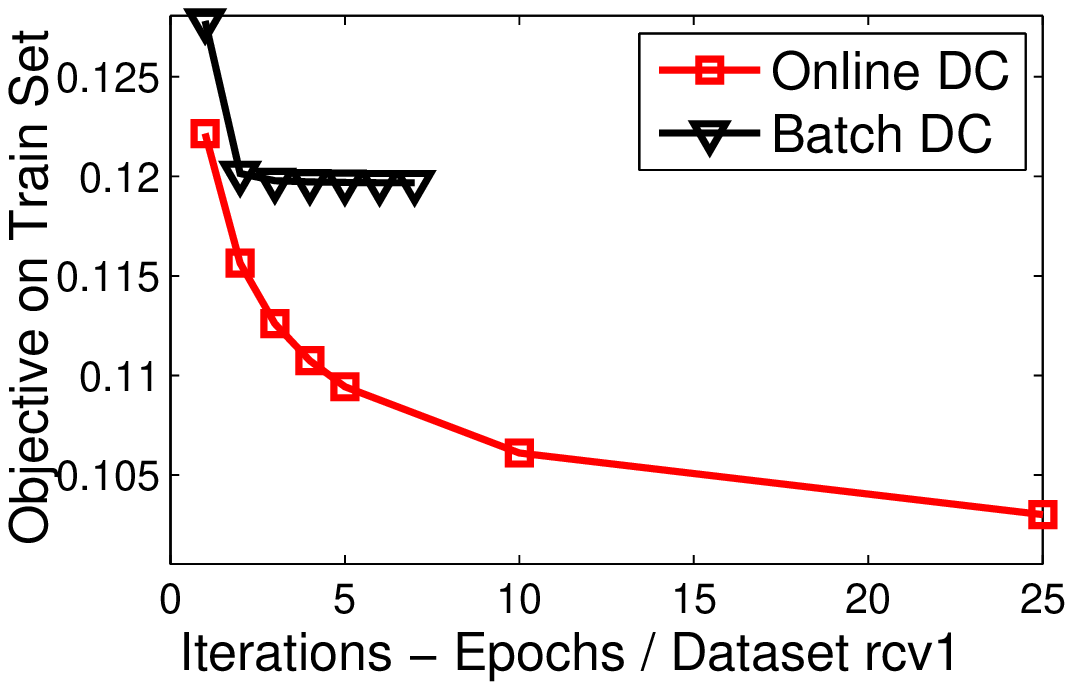}
   \includegraphics[width=0.4\linewidth]{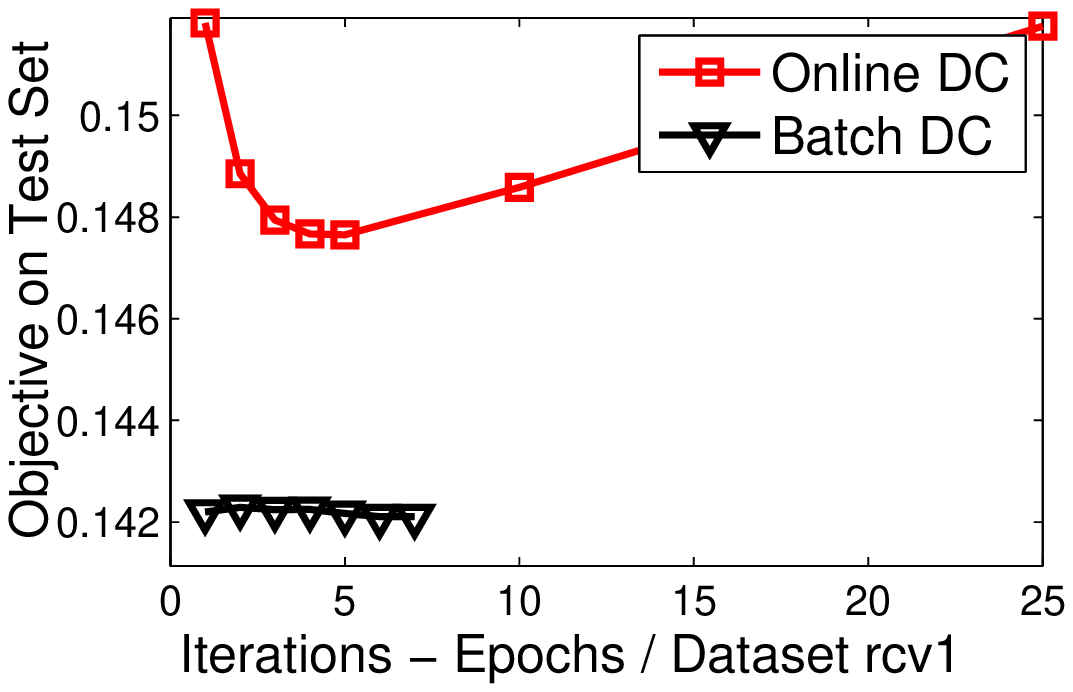} \\
   \includegraphics[width=0.4\linewidth]{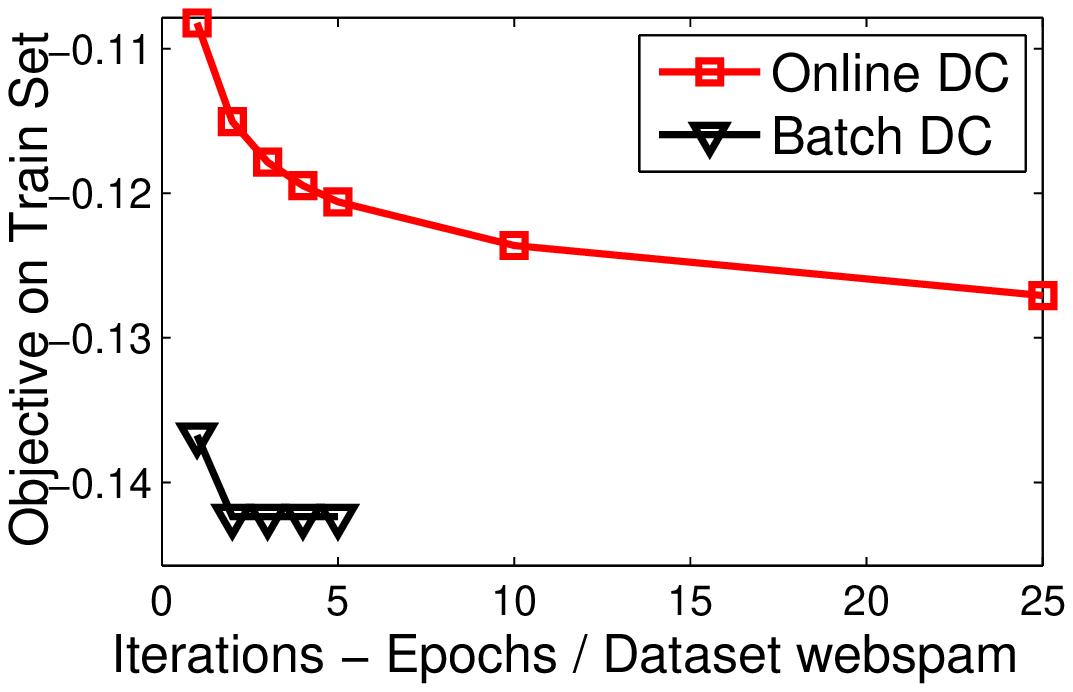}
   \includegraphics[width=0.4\linewidth]{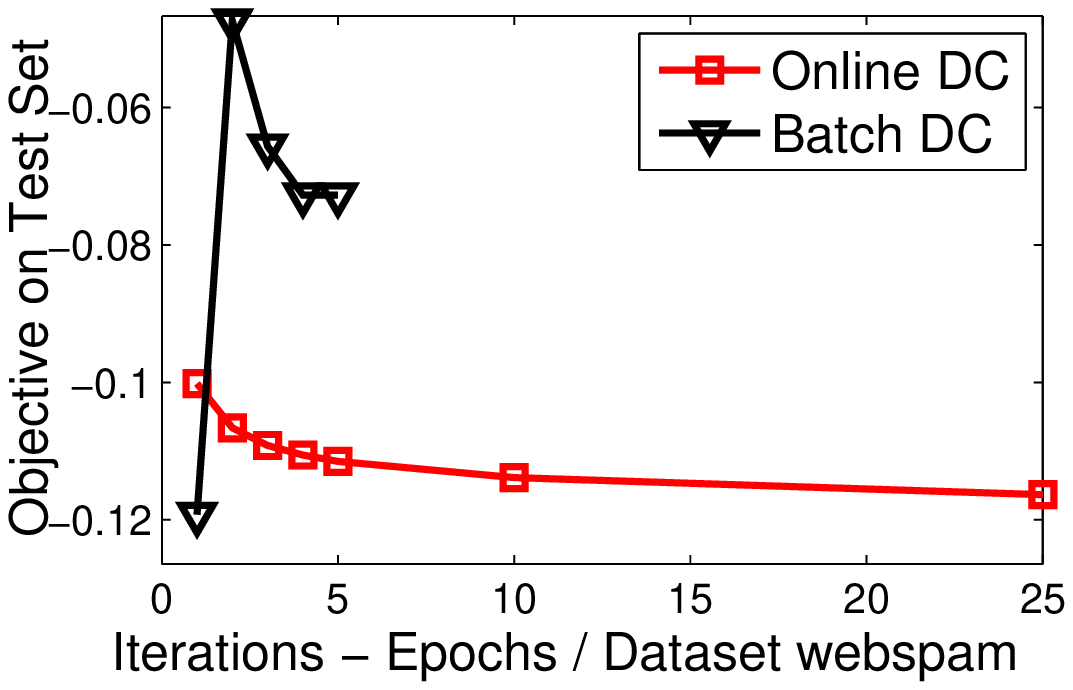}
   \caption{Comparison between batch and online DC programming, with low regularization for the datasets~\textsf{rcv1} and~\textsf{webspam}. Note that each iteration in the batch setting can perform several epochs.}\label{fig:exp_logb}
\end{figure}

\begin{figure}[hbtp]
   \centering
   \includegraphics[width=0.7\linewidth]{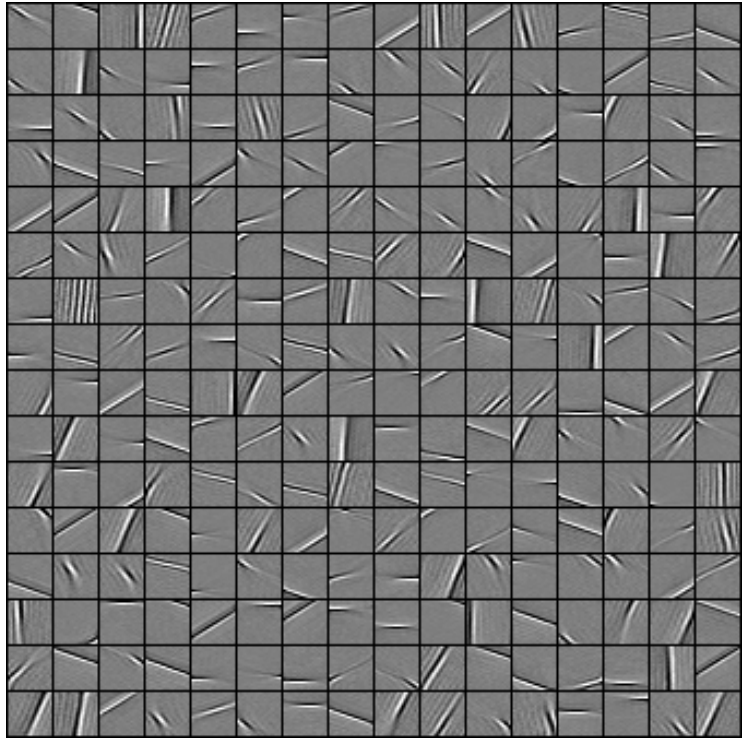}
   \caption{Dictionary obtained using the toolbox SPAMS~\cite{mairal7}.}
   \label{fig:exp_sp2}
\end{figure}

\begin{figure}[hbtp]
   \centering
   \includegraphics[width=0.7\linewidth]{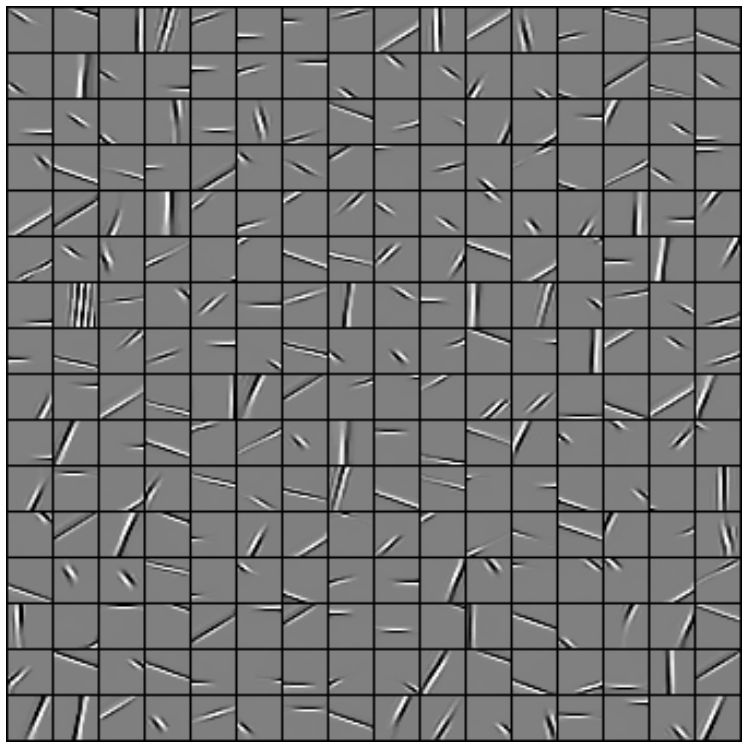}
   \caption{Sparse dictionary obtained by our approach, using the dictionary of Figure~\ref{fig:exp_sp2} as an initialization.}
   \label{fig:exp_sp3}
\end{figure}

\begin{figure}[hbtp]
   \centering
   \includegraphics[width=0.7\linewidth]{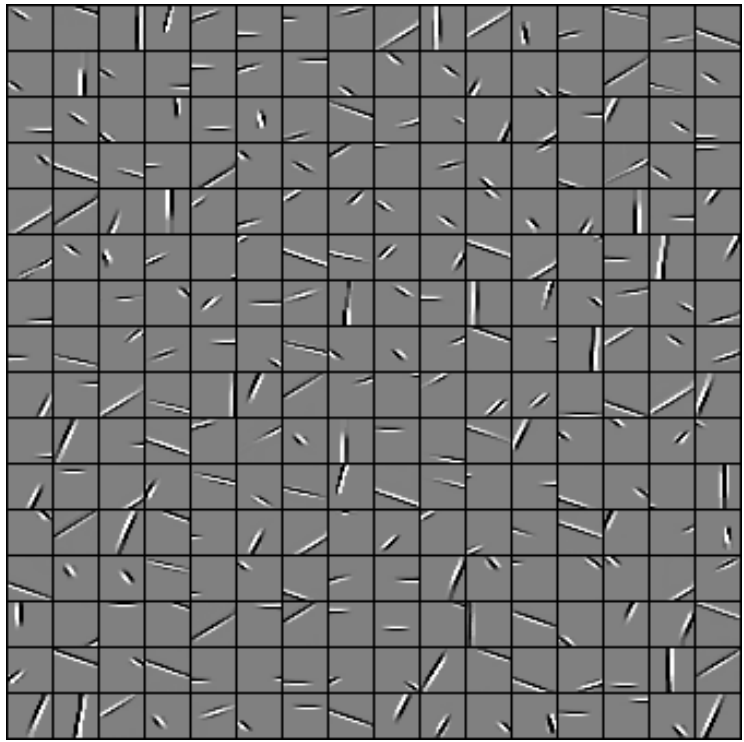}
   \caption{Sparse dictionary obtained by our approach, using the dictionary of Figure~\ref{fig:exp_sp2} as an initialization, and with a higher regularization parameter than in Figure~\ref{fig:exp_sp3}.}
   \label{fig:exp_sp4}
\end{figure}

\small{
\bibliographysup{abbrev,main}
\bibliographystylesup{plain}
}

\end{document}